\documentclass[wcp]{jmlr}

\usepackage{longtable}

\usepackage{booktabs}

\usepackage{amssymb,amsmath,color}
\usepackage{dsfont}
\usepackage{graphicx}
\usepackage{algorithm}
\usepackage{algpseudocode}
\usepackage{amsfonts}
\usepackage{nicefrac}
\usepackage[T1]{fontenc}

\newcommand{\loss}{\ell}
\newcommand{\decision}{\hat{\theta}}
 \newcommand{\generr}{\mathcal{E}}

\newcommand\cut[1]{}

\newcommand{\squishlist}{
   \begin{list}{$\bullet$}
    { \setlength{\itemsep}{0pt}      \setlength{\parsep}{3pt}
      \setlength{\topsep}{3pt}       \setlength{\partopsep}{0pt}
      \setlength{\leftmargin}{1.5em} \setlength{\labelwidth}{1em}
      \setlength{\labelsep}{0.5em} } }

\newcommand{\squishlisttwo}{
   \begin{list}{$\bullet$}
    { \setlength{\itemsep}{0pt}    \setlength{\parsep}{0pt}
      \setlength{\topsep}{0pt}     \setlength{\partopsep}{0pt}
      \setlength{\leftmargin}{2em} \setlength{\labelwidth}{1.5em}
      \setlength{\labelsep}{0.5em} } }

\newcommand{\squishend}{
    \end{list}  }

{}
\newtheorem{thm}{Theorem}{}
\newtheorem{prop}{Proposition}{}
{}
{}

\newcommand{\half}{\mbox{$\frac{1}{2}$}}

\newcommand{\sqr}[1]{\left[#1\right]}

\newcommand{\myexpect}{\mathbb{E}}

\newcommand{\calD}{\mbox{${\cal D}$}}

\newcommand{\data}{\calD}

\newcommand{\be}{\begin{equation}}
\newcommand{\ee}{\end{equation}}
\newcommand{\bea}{\begin{eqnarray}}
\newcommand{\eea}{\end{eqnarray}}
\newcommand{\beaa}{\begin{eqnarray*}}
\newcommand{\eeaa}{\end{eqnarray*}}

\DeclareMathOperator*{\argmin}{arg\,min}

\newtheorem{asm}{Assumption}[section]

\allowdisplaybreaks

\jmlrvolume{101}
\jmlryear{2019}
\jmlrworkshop{ACML 2019}

\title[A Generalization Bound for Online Variational Inference]{A Generalization Bound for Online Variational Inference}

  \author{\Name{Badr-Eddine Ch\'erief-Abdellatif} \Email{badr.eddine.cherief.abdellatif@ensae.fr}\\
  \addr CREST, ENSAE, Institut Polytechnique de Paris
  \AND
  \Name{Pierre Alquier} \Email{pierrealain.alquier@riken.jp}\\
  \addr RIKEN Center for AI Project, Tokyo, Japan
  \AND
  \Name{Mohammad Emtiyaz Khan} \Email{emtiyaz.khan@riken.jp}\\
  \addr RIKEN Center for AI Project, Tokyo, Japan
 }

\editors{Wee Sun Lee and Taiji Suzuki}

\begin{document}

\maketitle

\begin{abstract}
Bayesian inference provides an attractive online-learning framework to analyze sequential data, and offers generalization guarantees which hold even with model mismatch and adversaries. Unfortunately, exact Bayesian inference is rarely feasible in practice and approximation methods are usually employed, but do such methods preserve the generalization properties of Bayesian inference ? In this paper, we show that this is indeed the case for some variational inference (VI) algorithms. We consider a few existing online, tempered VI algorithms, as well as a new algorithm, and derive their generalization bounds. Our theoretical result relies on the convexity of the variational objective, but we argue that the result should hold more generally and present empirical evidence in support of this. Our work in this paper presents theoretical justifications in favor of online algorithms relying on approximate Bayesian methods.
\end{abstract}

\section{Introduction}
Bayesian methods, such as Kalman Filtering \citep{kalman1960}, Hidden Markov Model \citep{HiddenMarkovModel1966} and Particle Filtering \citep{DoucetParticleFiltering}, are popular methods to analyze sequential data. The posterior distribution provides a natural representation of the past information and can be updated sequentially using the Bayes rule whenever new data is available. Generalizations of Bayesian inference, such as those that \emph{temper} the likelihood, offer good generalization guarantees \citep{banerjee2006bayesian, audibert2009fast,gerchinovitz2013sparsity}. Such bounds hold even when the model is misspecified or when an adversary manipulates the stream of data. These generalization bounds are in fact very similar and sometimes even identical to the ones obtained by online learning methods commonly used in the optimization community \citep{cesa2006prediction}. The Bayesian principle offers a new perspective which can be used to advance online-learning methods used in areas such as convex optimization, machine learning, reinforcement learning, continual learning, and lifelong learning. 

Unfortunately, exact Bayesian inference is computationally challenging in cases where the normalizing constant of the posterior distribution is a high-dimensional integral.
Approximation methods, such as variational inference (VI) \citep{VIJordan1999} and expectation propagation \citep{MinkaEP}, can dramatically reduce the computation cost and enable application of the Bayesian principle to large-scale problems. Despite concerns about their approximation error, these methods have extensively been applied to many machine-learning problems where they show satisfactory performance in practice \citep{Blei2006TopicModels,hoffman2013stochastic, kingma2013auto}.

The practical success of such approximation methods points to the gap between the theory and practice. A few recent works have established generalization bounds of the approximation methods such as variational inference, but these are restricted to the batch or offline setting~\citep{Tempered,Plage,Chicago}. Extending such results to the online setting, without making strong assumption about the model mismatch and adversaries, is the main focus of this paper.

We propose online version of variational inference with tempered likelihoods, and derive new generalization bound, which has very similar form to the bound of exact Bayesian inference. Unlike existing proof techniques, our proof extend to the case when approximations are used instead of the exact Bayesian update. Our derivation relies on the convexity of the variational objective. This covers a few important cases, but can be limiting. We argue that the generalization bound is likely to hold more generally, and present empirical evidence in support of these arguments. Our work takes a step towards establishing the generalization properties of online approximate Bayesian methods.

\subsection{Related works}
Variational inference is extremely popular in statistics and machine learning, yet its theoretical properties are not investigated until recently.
Generalization bounds for generalized versions of variational approximations are derived in~\cite{alquier2016properties,cottet2018}. Similarly, Bernstein-von Mises' theorems for variational approximations in parametric models are proved in~\cite{wang2018frequentist}, while concentration of the posterior in general models is studied in~\cite{Tempered,sheth2017excess,Plage,Chicago,cherief2018consistency,cherief2018consistency2,jaiswal2019asymptotic}. These works show that variational approximations does enjoy the same consistency properties as the posterior distribution under general conditions. All of these results however only apply to the batch setting and their extension to the online setting is not straightforward.

It is known that the Bayesian approach leads to good online predictions for a stream of data; see~\cite{banerjee2006bayesian}, and~\cite{cesa2006prediction,audibert2009fast,gerchinovitz2013sparsity} for generalized posteriors in machine learning. However, there are only a few attempts to study the online properties of variational inference, and the proofs used in~\cite{cesa2006prediction} cannot easily be extended to online variational inference. 

Generalization bounds for online approximations of the posterior are studied in~\cite{dunson2013}, but the algorithms analyzed there are different from the ones used in practice and the feasibility of these algorithms is not proven. Recently \cite{nguyen2017online} give some results, but the order of magnitude of the bounds are not explicitly written and in many contexts it is not clear that the bound will even be small enough to ensure consistency. Even though stochastic/online versions of variational inference are known to give good results in practice~\citep{sato2001online, hoffman2010online, wang2011online,hoffman2013stochastic,khan2017conjugate,nguyen2017variational,khan2018,EmtiForComputations,zeno2018}, existing works have not been able to derive theoretical results confirming their generalization properties. Our results fill this gap between theory and practice for some types of variational approximations obtained with specific types of online algorithms.

\section{Generalization Properties of Bayesian Inference for Online Learning}
\label{section-Bayes}

Given a stream of data, the goal of online learning is to learn to make good decisions, estimations, or predictions on future data examples. 
The quality of such decisions is defined with a loss function $\loss(\data_t,\decision_t)$, denoted by $\loss_t(\decision_t)$ for brevity, where $\data_t$ is the data at time $t$ and $\decision_t$ is a quantity computed using the past data, i.e., $\data_{1:(t-1)} := \{ \data_1,\data_2, \ldots,\data_{t-1} \}$. 
This definition of the loss includes popular supervised and unsupervised learning methods. For example, in maximum-likelihood training of a parameterized model $p_\theta$, $\decision_t$ is the parameter estimate and the loss is $\loss_t(\theta) := -\log p_{\theta}(\data_t)$.
Similarly, for a classification task with input-output pair $\data_t:=(X_t,Y_t)$, the loss could be the hinge loss $\loss_t(\theta)=(1-Y_t f_{\theta}(X_t))_+$ with a classifier $f_\theta$. In the whole paper, we assume that $\theta\mapsto \loss_t(\theta)$ is convex.
By using losses $\loss_t$ until time $t$, our ultimate goal is to find a $\theta_t$ which is as close as possible to the minimizer $\theta^*$ of the generalization error 
$ \generr_*(\theta) = \mathbb{E}_{\mathcal{D}\sim P_*}[\loss(\data,\theta)] $ where $P_*$ is the true distribution of the data. We would want to do this without many strong assumptions such as assuming the data stream to be i.i.d., or the absence of adversaries. 

Since $\generr_*$ is unavailable at time $t$, to ensure the quality of $\decision_t$, online-learning algorithms aim at minimizing the cumulative error $\sum_{i=1}^t \loss_i(\hat{\theta}_t)$ until time $t$. Many algorithms are known with bounds on the \emph{regret} of the decision $\hat{\theta}_t$ , that is the gap in the cumulative error and the minimal cumulative error that could have been reached with a {\it fixed} parameter:
\begin{align}
 \sum_{t=1}^T \loss_t(\hat{\theta}_t) - \inf_{\theta\in\Theta} \sum_{t=1}^T \loss_t(\theta). 
 \end{align}
Bounds on this quantity are known as \emph{regret} bounds, e.g., see~\cite{cesa2006prediction,bubeck,shalev2012online,HazanOnlineConvexOptimization}.
Fortunately, bounding the regret also leads to upper bounds on the generalization gap, e.g., by using the average $\bar{\theta}_T=\frac{1}{T}\sum_{t=1}^T \hat{\theta_t}$ we can bound the gap $\generr_*(\bar{\theta}_T)-\generr_*(\theta^*)$. Due to such properties, regret bounds are useful to study generalization properties of an online algorithm.
Moreover, the bound holds with very little assumptions on the data and is valid when the data is not i.i.d. and even when it is corrupted by an adversary.

For online learning, Bayesian inference algorithms have good generalization properties, e.g.,   
the following \emph{tempered} posterior distribution introduced by~\cite{Vovk:1990:AS:92571.92672,littlestone1994weighted} has a controlled regret: 
\begin{align}
 p_t^{\eta}(\theta) := \frac{1}{\mathcal{Z}_t^\eta} \pi(\theta) e^{-\eta \sum_{i=1}^{t-1} \loss_t(\theta)} 
 \label{eq:temperedBayes}
\end{align}
where $\eta>0$ is a learning rate, $\pi$ is a prior distribution, and $\mathcal{Z}_t^\eta$ is the normalizing constant of the posterior distribution.
Each loss $\loss_t$ here can be interpreted as the log-likelihood of a data example $\data_t$. When the loss is indeed equal to $-\log p_\theta(\data)$ and $\eta=1$, the above algorithm is equivalent to Bayesian inference whose generalization properties are usually established under the assumption of no model mismatch (e.g., see~\cite{ghosal2017fundamentals}). 
The tempered version $\eta<1$ can be shown to generalize well even when the model is misspecified~\citep{grunwaldmisspecifiation} or when an adversary manipulates the stream of data.
Such tempered versions have also been studied in depth in the machine-learning literature by using the PAC-Bayesian bounds~\citep{shawe1997pac,mcallester1999some,MR2483528,seldin2010pac,suzuki2012pac,seldin2011pac,cuong2013generalization,germain2016pac,giulini2017,guedj2019primer,tsuzuku2019normalized}.

In the online-learning literature, the regret bound of this algorithm has been studied extensively under a variety of names, e.g., algorithms such as multiplicative update, weighted majority algorithm, exponentially weighted aggregation (EWA) are all specific cases of tempered Bayesian inference. Algorithm \ref{algo-ewa} shows a pseudo-code for EWA which performs tempered Bayesian inference in an online fashion (Step 3 implements Equation \eqref{eq:temperedBayes}). Below, we state a theorem which shows an example of regret bound\footnote{In online-learning literature such results are usually stated for finite decision space, e.g., see similar results for EWA in~\cite{cesa2006prediction}. The result above holds for a more general continuous setting but under a bounded loss.}, proved in
Theorem 4.6~in \cite{audibert2009fast} for the algorithm shown in Algorithm \ref{algo-ewa}.
\begin{thm}
\label{thm-EWA}
Assuming that the loss is bounded, i.e., $0\leq\loss_t (\theta)\leq B, \,\forall \data_t, \theta$, the cumulative regret has the following upper bound when $\decision_t = \myexpect_{\theta \sim p^\eta_t}[\theta]$ is the posterior mean:
\begin{align}
\label{eq-regret-EWA}
\sum_{t=1}^T \loss_t(\hat{\theta}_t) \leq \inf_{p\in\mathcal{S}} \left\{ \mathbb{E}_{\theta\sim p}\left[\sum_{t=1}^T \loss_t(\theta)\right] + \frac{\eta B^2 T}{8} + \frac{\mathcal{K}(p,\pi)}{\eta} \right\}
\end{align}
where $\mathcal{S}$ is the set of all probability distributions over $\Theta$ and $\mathcal{K}$ is the K\"ullback-Leibler (KL) divergence.
\end{thm}
\begin{algorithm}[t]
\caption{Tempered Bayesian Inference, a.k.a Exponentially Weighted Aggregration}
\label{algo-ewa}
\begin{description}
\item[Input] Learning rate $\eta>0$, prior $\pi(\theta)$, $p_1^\eta \leftarrow \pi$.
\item[For] $t=1,2,3,\dots$,
\begin{description}
\item[1.] $\hat{\theta}_t \leftarrow \mathbb{E}_{\theta\sim p_t^\eta}(\theta) $,
\item[2.] Observe $\data_t$ to suffer a loss $\loss_t(\hat{\theta}_t)$.
\item[3.] Update $p_{t+1}^\eta(\theta) \propto p_{t}^\eta(\theta) \exp \sqr{-\eta \loss_t(\theta) } $.
\end{description}
\end{description}
\end{algorithm}
A proof is given in Appendix \ref{proofs} for the sake of completeness.

The above regret bound is useful to derive explicit bounds in expectation on the generalization error $\generr_*$ of an estimator that is defined as the average decision $\bar{\theta}_T := \sum_t \decision_t/T$.
For example, we can show that, when a classical prior mass condition\footnote{The exact condition is that the prior $\pi(\theta)$ has mass bigger than $\epsilon^d$ on an $\epsilon$-ball around $\theta^*$ for some $d$.} on the prior is satisfied and when $\data_t$ are actually independent and identically distributed from $P_*$, the generalization error has the following bound:
\begin{align}
\label{eq-generr-EWA}
\myexpect_{\mathcal{D}_{1:T}\sim P_*}\left[ \generr_*(\bar{\theta}_T)\right]
\leq  \generr_*(\theta^*) + B \sqrt{\frac{d}{2T}\log\left(\frac{T}{d}\right)}
\end{align}
for some well-chosen $\eta\sim \sqrt{d/T}$ and $d>0$ is a complexity measure of the parameter space (often the dimension).
This bound shows that when $\data_t$ are i.i.d. from $P^*$ then Bayesian inference achieves generalization error at a rate $\sqrt{d/T}$. An exact statement and a complete proof are given in Theorem~\ref{thm-dim} Subsection~\ref{subsection-otb} in the appendix. The proof is based on a technique called \emph{online-to-batch} analysis. 
Similar bounds can be derived even for the cases when the model is misspecified and an adversary is present.

The regret bound derived in Theorem \ref{thm-EWA} assumes that $p_t^\eta$ is computed exactly, which is extremely challenging and many a times infeasible.
The difficulty arises due to the computation of $\mathcal{Z}_t^\eta$ which is a high-dimensional integral when the space of $\theta$ is large.
Approximate Bayesian inference methods approximate the integral by finding an approximation of $p_t^\eta$ in a restricted family of distributions $\mathcal{F}=\{q_\mu,\mu\in \mathcal{M}\}$, e.g., Gaussian distribution with $\mu$ being the mean and variance.
Our focus in this paper is to derive bounds similar to Theorem \ref{thm-EWA} for approximate Bayesian inference methods.

Unfortunately, deriving similar bounds as Theorem \ref{thm-EWA} for approximate inference is not possible using existing proof techniques. This is because these techniques do not work when $p_t^\eta$ and $\mathcal{S}$ in \eqref{eq-regret-EWA} are replaced by $q_{\mu_t}$ and $\mathcal{M}$ respectively. As shown in Appendix \ref{proofs}, these proofs rely on cancellation of many terms in a telescoping sum. This
cancellation does not take place when an approximation is used instead, and the error accumulates making the regret bound practically useless. In this paper, we solve this problem using a different proof for tempered, online variational inference algorithms discussed in the next section.

\section{Online Variational Inference}
\label{section-onlinevi}
In this section, we introduce approximate Bayesian inference methods that can obtain tractable approximations in an online fashion. The methods available in the approximate inference literature are not always suitable for our purpose. Therefore, we present modifications of those methods that lead to feasible online variants of the Bayesian update shown in \eqref{eq:temperedBayes}. 
%
%
To simplify the notation, we will denote the expectation of the loss under an approximation $q_{\mu}(\theta)$ by $\bar{L}_t(\mu):=\mathbb{E}_{\theta \sim q_\mu}[\loss_t(\theta)]$.

\subsection{Sequential Variational Approximation}
An advantage of variational inference is that it can be directly written as a constrained optimization version of Bayesian inference. To see this we first note that the posterior given in \eqref{eq:temperedBayes} can be obtained by solving the following optimization problem \citep{dai2016provable}:
\begin{align}
p_{t+1}^\eta(\theta) = \argmin_{p\in\mathcal{S}} \left\{  \mathbb{E}_{\theta \sim p} \bigg[ \sum_{i=1}^{t}  \loss_{i}(\theta) \bigg] + \frac{\mathcal{K}(p,\pi)}{\eta} \right\} \nonumber
\end{align}
We can obtain an approximation by simply restricting the set $\mathcal{S}$:
\begin{align}
\label{definitionVB}
 q_{\mu_t} := \argmin_{\mu \in \mathcal{M}} \left\{  \mathbb{E}_{\theta \sim q_\mu} \bigg[ \sum_{i=1}^{t-1}  \loss_i(\theta) \bigg] + \frac{\mathcal{K}(q_\mu,\pi)}{\eta} \right\}
\end{align}
where the set $\mathcal{M}$ is the set of parameters for the set $\mathcal{F}:=\{q_\mu,\mu\in \mathcal{M}\}$. The above approximation therefore is a variational approximation of the exact Bayesian inference.

Unfortunately, the update \eqref{definitionVB} may not be feasible in practice. 
The Bayesian update of \eqref{eq:temperedBayes} takes a convenient form where update of $p_{t+1}^\eta$ can be written in terms of $p_t^\eta$; see line 3~in Algorithm \ref{algo-ewa}. For update \eqref{definitionVB}, this is not possible in most cases, i.e., we cannot express the optimization problem for $q_{\mu_{t+1}}$ in terms of $q_{\mu_t}$. Typically, one need to store all the past data examples $\data_i$ and recompute their gradients, and then run the optimizer until it converges.
This can be very expensive, especially for large $t$.

We propose a sequential version which solves these problems by using an approximation. We follow the ideas used in online gradient algorithms, e.g., such as those used in~\cite{shalev2012online}, and replace $\mathbb{E}_{\theta \sim q_\mu} [\loss_i(\theta)] = \bar{L}_i(\mu) \approx \mu^T \nabla_\mu \bar{L}_i(\mu_i)$. This leads to
\begin{align}
\mu_{t+1} = \argmin_{\mu\in \mathcal{M}} \Biggl[ \sum_{i=1}^{t} \mu^T \nabla_\mu \bar{L}_i(\mu_i) + \frac{\mathcal{K}(q_\mu,\pi)}{\eta} \Biggr].
\label{eq:sva_opt}
\end{align}

Note that the gradients in the approximation are computed at the past $\mu_i$, rather than the current one $\mu_t$. This results in an algorithm summarized in Algorithm
\ref{algo-onlineVB1} which we call sequential variational approximation (SVA). When computing the gradient of the KL divergence term is feasible, this algorithm can be cheaply performed.

\begin{algorithm}[t]
\caption{Online Variational Inference}
\label{algo-onlineVB1}
\begin{description}
\item[Input] Learning rate $\eta>0$, a prior $\pi(\theta) \in \mathcal{F}$, $q_{\mu_1} \leftarrow \pi$.
\item[For] $t=1,2,3,\dots$,
\begin{description}
\item[1.] $\hat{\theta}_{t}\leftarrow\mathbb{E}_{\theta\sim q_{\mu_t}}[\theta]$,
\item[2.] Observe $\data_t$ to suffer a loss $\loss_t(\hat{\theta}_t)$.
\item[3.] Update depending on the type of algorithm.
   \begin{enumerate}
   \item[a)] For SVA, solve \eqref{eq:sva_opt}.
   \item[b)] For SVB, solve \eqref{eq:svb_opt}.
   \item[c)] For NGVI, solve \eqref{eq:ngvi}.
   \end{enumerate}
\end{description}
\end{description}
\end{algorithm}

\subsection{Streaming Variational Bayes}
An alternative approach is to remove the term $\mathcal{K}(q_\mu, \pi)$ since $\pi$ is already included in $q_{\mu_t}$:
\begin{align}
\mu_{t+1}
= \argmin_{\mu\in \mathcal{M} } \Biggl[ \mu^T \nabla_\mu \bar{L}_t(\mu_t) + \frac{\mathcal{K}(q_\mu,q_{\mu_t})}{\eta} \Biggr]. 
\label{eq:svb_opt}
\end{align}
This step, contained in Algorithm \ref{algo-onlineVB1}, is tractable whenever computing the gradient of the KL term is feasible, e.g., when the expectation parameterization is used. This type of update has been proposed in many recent works, e.g., \cite{nguyen2017online}, \cite{zeno2018}. These updates can be seen as a special case of \cite{StreamingVB}. Due to this connection, we call this algorithm streaming variational Bayes (SVB).

\subsection{Natural Gradient Variational Inference}
The algorithm described in the previous sections are closely related to existing natural-gradient variational inference (NGVI) algorithm~\citep{sato2001online, hoffman2013stochastic, khan2017conjugate}. These algorithms are typically applied for \emph{stochastic} learning but can be easily modified for online setting. We will consider the method of \cite{khan2017conjugate} because it applies to the most general setting (other methods require strong \emph{conjugacy} assumptions on the loss $\loss_t(\theta)$ and prior $\pi(\theta)$). The NGVI algorithm is typically applied to obtain exponential-family approximations, but as we will show the updates are similar to our SVA algorithm which also reveals a more general way of implementing these algorithms in the online setting.

The advantage of using NGVI for online learning is that it obtains closed-form updates for $q_{\mu_{t+1}}$ which can be expressed in terms of $q_{\mu_t}$. This is done by exploiting the expectation parameterization\footnote{Expectation parameters are expectations of the sufficient statistics, e.g., Gaussian approximation has two expectation parameters: mean vector and correlation matrix respectively.} of the exponential family.
Throughout this section, we denote the expectation parameter by $\mu$ and natural parameterization of the exponential family by $\lambda$. 
\cite{khan2017conjugate} propose the following update\footnote{The exact update proposed in \cite{khan2017conjugate} is written differently but can be shown to be equivalent to \eqref{eq:ngvi}. This can be done by using their Lemma 1 and setting $1/\alpha:= 1/\beta- 1/\eta$ where $\beta$ is the step-size used in their paper. We use this form since it makes it easier to establish connections to SVA.} in the expectation-parameter space:
\begin{align}
\min_{\mu\in \mathcal{M}} \Biggl[ \mu^T \nabla_\mu \bar{L}_t(\mu_t) + \frac{\mathcal{K}(q_\mu,\pi)}{\eta} + \frac{\mathcal{K}(q_\mu,q_{\mu_t})}{\alpha} \Biggr],
\label{eq:ngvi}
\end{align}
where $\alpha>0$ is a step size.
The difference from \eqref{eq:sva_opt} is that now the linear term does not contain a sum over all past examples $i$, rather only the current one.
Instead, we add another KL divergence term which contains the past information in the previous approximation $q_{\mu_t}$.
Therefore, NGVI algorithm, summarized in Algorithm \ref{algo-onlineVB1}, employs a different way to add the past information, but as we show next, it results in a very similar update as SVA. In the appendix, we provide a closed-form solution to~\eqref{eq:ngvi}.

\subsection{Example: Mean-Field Gaussian VI}
We now give a concrete example of the algorithms introduced in this section.
We will use the mean-field Gaussian VI where $\mathcal{F}$ is the class of all Gaussian approximations with diagonal covariance matrix. We denote the mean vector of the Gaussian by $m=(m_1,\dots,m_d)^T$ and the diagonal of the covariance matrix by $\sigma^2=(\sigma_1^2,\dots,\sigma_d^2)^T$. To derive the updates for SVA and SVB, we used $\mu = \{m, \sigma\}$ while for NGVI we used the expectation parameters $\mu = \{m, m^2 + \sigma^2\}$. (Here, and until~\eqref{algo-log} below, the squares and multiplications on vectors are to be understood componentwise). We also assume the prior $\pi(\theta)$ to be a Gaussian with mean 0 and variance $s^2 I_d$ where $I_d$ is the identity $d\times d$ matrix.

Denoting the gradients $\bar{g}_{m_t} := \frac{\partial \bar{L}_t}{\partial m}$ and $\bar{g}_{\sigma_t} := \frac{\partial \bar{L}_t}{\partial \sigma}$, we give the update for each method below (here $h(x):=\sqrt{1+x^2}-x$, applied componentwise for vector inputs):
\begin{align}
\textrm{SVA:} \,\, m_{t+1} & \leftarrow m_{t} - \eta s^2 \bar{g}_{m_t} , \quad\quad g_{t+1} \leftarrow g_t + \bar{g}_{\sigma_t} \nonumber,
\\
\sigma_{t+1} & \leftarrow  h\left(\half\eta s g_{t+1}\right) s , \\
\textrm{SVB:} \,\, m_{t+1} &\leftarrow m_{t} - \eta \sigma_t^2  \bar{g}_{m_t} , \nonumber \\
\sigma_{t+1} &\leftarrow \sigma_t h\left( \half\eta \sigma_t \bar{g}_{\sigma_t} \right) . \label{algo-log} 
\end{align}


\section{Generalization Bounds for Online VI}
\label{section-bounds}
In this section, we present regret bounds for online VI algorithms discussed in the previous section. Our bounds take similar form to the one presented in Theorem \ref{thm-EWA}, and can be used to obtain generazation bounds similar to \eqref{eq-generr-EWA}. Our proofs require convexity of $\bar{L}_t(\mu) := \myexpect_{q_\mu} [\loss_t(\theta)]$ with respect to $\mu$, which is a strong assumption. Due to this we are able to derive bounds for SVA and SVB. We expect our bound to hold for NGVI too, due to its similarity to
SVA. Specifically, all of our results use the following minimal assumption.
\begin{asm}
\label{asm-l} $\bar{L}_t$ is $L$-Lipschitz and convex.
\end{asm}
Some results require the following stronger assumption.
\begin{asm}
\label{asm-s-c}
$\bar{L}_t$ is $H$-strongly convex where $H>0$, i.e., for any two $\mu,\mu'\in\mathcal{M}$, the following holds:
$$ 
\bar{L}_t(\mu') - \bar{L}_t(\mu) \geq (\mu'-\mu)^T \nabla \bar{L}_t(\mu) + \frac{H}{2} \|\mu'-\mu\|^2.
$$
\end{asm}
Finally, some results also require strong convexity for KL.
\begin{asm}
\label{asm-k}
The KL divergence $\mu\mapsto \mathcal{K}(q_\mu,q_{\mu_1})$ is $\alpha$-strongly convex.
\end{asm}
All of these assumption depend heavily on the parametrization of $\{q_\mu,\mu\in M\}$. For some parameterization, these assumptions do hold although such cases are limited. For example, for Gaussian approximations and convex $\loss$, the assumptions are satisfied, as pointed out by \cite{challis2013gaussian}. This result has recently been extended by~\cite{ConvexDomke} to more generals \emph{location-scale} family. We give a formal statement
below.
\begin{prop}[Theorem 1 in~\cite{ConvexDomke}]
 \label{prop-cvx}
Assuming that $q_\mu$ belongs to a location-scale family $\mathcal{F} = \left\{q_{m,C} \right\} $ where $m$ is a d-length vector and $C$ is a $d\times d$ matrix with $q_{m,C}(\theta) = [{\rm det}(C)]^{-1/2} \psi(C^{-1/2}(\theta-m)) $
 for some fixed density $\psi$, then $\bar{L}_t$ is convex. Moreover when each $\theta \mapsto \loss_t(\theta) $ is $H$-strongly convex and $\psi$ is the density of a centered random variable with identity variance matrix, then Assumption~\ref{asm-s-c} is also satisfied.
\end{prop}
The results for Gaussian approximation can be obtained as a special case.
\begin{prop}
 \label{prop-lipschitz}
 Assume that $\theta \mapsto \loss_t(\theta) $ is $L'$-Lipschitz. Assume that we use the Gaussian approximation family $\mathcal{F} = \left\{q_{m,C}= \mathcal{N}(m,C^T C),(m,C)\in M  \right\} $, $M\subset  \mathbb{R}^d \times UT(d)$ where $UT(d)$ is the set of full-rank upper triangular $d\times d$ real matrices. Then $\bar{L}_t$ is $L$-Lipschitz with $L=2L'$.
\end{prop}

Finally, we remind the formula for the KL divergence between two Gaussian distributions. Let $q_{m,C}=\mathcal{N}(m,C^T C)$ for any $(m,C)\in\mathbb{R}^d \times UT(d)$. Then
\begin{equation*}
 \mathcal{K}(q_{m,C},q_{\bar{m},\bar{C}})
  = \frac{1}{2}  \bigg( (m-\bar{m})^T \bar{C}^T \bar{C} (m-\bar{m})
  + {\rm tr}[(\bar{C}^T\bar{C})^{-1} (C^T C)] 
  + \log\left(\frac{{\rm det}(\bar{C}^T\bar{C})}{{\rm det}(C^T C)}\right) -d \bigg)
\end{equation*}
is known to be strongly convex on $\mathbb{R}^d \times \mathcal{M}_C$ where $\mathcal{M}_C$ is a closed bounded subset of $UT(d)$. Thus, Assumption~\ref{asm-k} is satisfied with a Gaussian prior and a Gaussian approximation family.

We are now ready to state our regret bounds for SVA and SVB.

\subsection{Bounds for SVA}

\begin{thm}
\label{thm1}
Under Assumptions~\ref{asm-l} and~\ref{asm-k}, SVA has the following regret bound: 
\begin{align}
\sum_{t=1}^T \loss_t(\hat{\theta}_t) \leq \inf_{\mu\in \mathcal{M}} \Biggl\{\mathbb{E}_{\theta\sim q_\mu}\left[ \sum_{t=1}^T \loss_t(\theta)\right]  + \frac{ \eta L^2 T}{\alpha} + \frac{\mathcal{K}(q_\mu,\pi)}{\eta} \Biggr\}.
\end{align}
\end{thm}
The above bound is almost identical to the bound given in Theorem~\ref{thm-EWA} where we can replace $p$ by $q_{\mu}$, $\mathcal{S}$ by $\mathcal{M}$, the bound $B$ by the Lipschitz constant $L$, and factor of 8 by the strong convexity parameter $\alpha$. However, our proof of Theorem~\ref{thm1} is completely different from the one for Theorem~\ref{thm-EWA}. It relies on arguments from online convex optimization that can be found
in~\cite{shalev2012online,HazanOnlineConvexOptimization}. A detailed proof is given in Appendix \ref{proofs}.

Similar to the Bayesian update case discussed in Section \ref{section-Bayes}, using the online-to-batch analysis detailed in Appendix~\ref{subsection-otb}, we can show that the average $\bar{\theta}_T=(1/T)\sum_{t=1}^T \hat{\theta}_t$ satisfies
\begin{equation}
\myexpect_{\mathcal{D}_{1:T}\sim P_*} [ \generr_*(\bar{\theta}_T)]
\leq \inf_{\mu\in \mathcal{M}} \Biggl\{ \mathbb{E}_{\theta\sim q_\mu} [\mathcal{E}_*(\theta)]  + \frac{ \eta L^2 }{\alpha} + \frac{\mathcal{K}(q_\mu,\pi)}{\eta T} \Biggr\}.
\end{equation}
As an example consider the mean-field Gaussian approximation and assume that for any $\mathcal{D}$, $\loss(\mathcal{D},\cdot)$ is $L/2$-Lipschitz (note that these are the assumptions of Proposition~\ref{prop-lipschitz} ensuring that Assumption~\ref{asm-l} is satisfied). Then $\mathbb{E}_{\theta\sim q_\mu} [\mathcal{E}_*(\theta)] = \mathcal{E}_*(m) + {\|\sigma\| L}/{2}$ .
Therefore, given the expression of the KL-divergence between Gaussian distributions, taking a vector $\sigma$ with $\sigma_j=L\eta/(\alpha\sqrt{d})$, $\eta=(1/L)\sqrt{\alpha d \log(T/d) /T}$, and considering only the regret with respect to bounded means $m$ leads to
$$
\myexpect_{\mathcal{D}_{1:T}\sim P_*} [ \generr_*(\bar{\theta}_T)] \leq
\inf_{m\in [-\bar{M},\bar{M}]^d} \mathcal{E}_*(m) + (1+o(1))\frac{2L}{\alpha} \sqrt{ \frac{d\log \left( dT\right)}{T} }.
$$
This again is very similar to the generalization error shown in \eqref{eq-generr-EWA}.

\subsection{Bounds for SVB}
Similarly to the SVA case, we can derive a regret bound, however our proof only applies to the Gaussian case. For this case, we require a dynamic learning $\eta_t$. We use a different learning rate for each element of $\theta_j$ which we denote by $\eta_{t,j}$. The result also works for a bounded parameter space $\mathcal{M}=\mathcal{M}_m\times \mathcal{M}_\sigma$ that will imply a projection step in addition to the update in~\eqref{algo-log}:
\begin{align*}
\textrm{SVB:} \,\, m_{t+1} & \leftarrow \Pi_{\mathcal{M}_m} \left[m_{t} - \eta \sigma_t^2  \bar{g}_{m_t}\right] , \nonumber \\
\sigma_{t+1} &\leftarrow  \Pi_{\mathcal{M}_\sigma} \left[ \sigma_t h\left( \half\eta \sigma_t \bar{g}_{\sigma_t} \right) \right].
\end{align*}
where $\Pi_{\mathcal{M}_m}$ and $\Pi_{\mathcal{M}_\sigma}$ denote the orthogonal projection on $\mathcal{M}_m$ and $\mathcal{M}_\sigma$ respectively. The following theorem states the result.
\begin{thm}
\label{thm3}
We consider the mean-field Gaussian family $q_\mu = \mathcal{N}(m,{\rm diag}(\sigma^2))$ and $\mathcal{M}=\mathcal{M}_m\times \mathcal{M}_\sigma $ where $\mathcal{M}_m$ and $\mathcal{M}_\sigma$ are closed, bounded, convex subsets of $\mathbb{R}^d$ and $\mathbb{R}^d_+$ respectively, and $0\in \mathcal{M}_{\sigma}$. Define $
D^2 = \sup \left\{ \|m-m'\|_2^2 + \|\sigma\|^2 , m,m' \in \mathcal{M}_m, \sigma \in \mathcal{M}_\sigma\right\}
$.
Then, under Assumption~\ref{asm-l}, with the choice $\eta_{t,j}=\frac{D\sqrt{2}}{L} \frac{1}{\sqrt{t}\sigma_{t,j}^2}$ we get:
\begin{align}
\sum_{t=1}^T \loss_t(\hat{\theta}_t)
\leq \inf_{\theta\in \mathcal{M}_m}  \sum_{t=1}^T \loss_t(\theta) + DL \sqrt{2T} .
\end{align}
Under Assumptions~\ref{asm-l} and \ref{asm-s-c}, the choice $\eta_t = 2/H t\sigma_t^2$ leads to:
\begin{align}
\sum_{t=1}^T \loss_t(\hat{\theta}_t)
\leq \inf_{\theta \in \mathcal{M}_m}  \sum_{t=1}^T \loss_t(\theta)  + \frac{L^2(1+\log T)}{H}.
\end{align}
\end{thm}
Here again the results are similar to the Bayesian inference case but now expressed in terms of the parameters $\mu$ instead of expectations.

A similar bound on the generalization error can also be proved. Define $\bar{\theta}_T=(1/T)\sum_{t=1}^T \hat{\theta}_t$. Here, the online-to-batch analysis directly leads to
$$
\myexpect_{\mathcal{D}_{1:T}\sim P_*} [ \generr_*(\bar{\theta}_T)]
\leq \inf_{\theta\in \mathcal{M}_m} \mathcal{E}_*(\theta)  + \frac{DL\sqrt{2}}{\sqrt{T}}
$$
in the convex case and
$$
\myexpect_{\mathcal{D}_{1:T}\sim P_*} [ \generr_*(\bar{\theta}_T)]
\leq \inf_{\theta\in \mathcal{M}_m} \mathcal{E}_*(\theta)  +  \frac{L^2(1+\log T)}{HT}
$$
in the strongly convex case.

Note the in the online optimization setting studied in~\cite{shalev2012online}, it is usual to optimize on Euclidean balls. Here, $M_m=\{m\in\mathbb{R}^d:\|m\|\leq \bar{M}\}$ and $M_\sigma=\{\sigma\in\mathbb{R}^d_+ :\|\sigma\|\leq \bar{S}\}$ leads to $D=4 \bar{M}^2 + \bar{S}^2$ leads to dimension-free bounds.

On the other hand, the choice  $M_m=[-\bar{M},\bar{M}]^d$ and $M_\sigma = [0,\bar{S}]^d$ implies $D^2=d(4 \bar{M}^2 + \bar{S}^2)$, and so the bound in the convex case is
\begin{equation*}
\myexpect_{\mathcal{D}_{1:T}\sim P_*} [ \generr_*(\bar{\theta}_T)]
\leq \inf_{\theta\in \mathcal{M}_m} \mathcal{E}_*(\theta)  + \frac{L\sqrt{2d(4 \bar{M}^2 + \bar{S}^2)}}{\sqrt{T}}
\end{equation*}
and its dependence in $d$ is the same as in the bound on SVA.

\subsection{Generalization}

We expect our bounds to hold for NGVI as well. When expectation parameterization is used, the assumptions are satisfied only in very limited models. This is because the result of Proposition \ref{prop-cvx} and \ref{prop-lipschitz} do not directly apply to expectation parameterization. However, the NGVI update shown in \eqref{eq:ngvi} can be applied in other parameterization as well, in which case some of our result can be extended to NGVI too.

\section{Experiments}
\label{section-expe}

In this section, we conduct experiments on real and simulated datasets, in classification and linear/nonlinear regression. The objective is twofold: check the convergence of SVA/SVB, with and without the convexity assumption on $\bar{L}_t$, and compare SVA, NGVI and SVB.

\subsection{Experimental setup}

We compare the empirical performance of the algorithms we present in this paper through classification and regression tasks on several toy and real-world datasets. We also include the classical online gradient descent and the online gradient descent on the expected loss as benchmarks. Please refer to Appendix \ref{Appendix:OGA-EL} for more details on these algorithms. In the following, OGA will stand for the classical online gradient descent while OGA-EL for the OGA on the expected loss (Algorithm \ref{algo-onlineVB2}). We recall that SVA, NGVI and SVB respectively refer to the sequential variational approximation (\ref{eq:sva_opt}), natural gradient variational inference (\ref{eq:ngvi}) and streaming variational Bayes (\ref{eq:svb_opt}).

\textbf{Binary classification} 
We consider first a classification problem. At each round $t$ the learner receives a data point $x_t \in \mathbb{R}^d$ and predicts its label $y_t\in\{-1,+1\}$ using $\left<x_t,\theta_t\right>$. The adversary reveals the true value $y_t$, then the learner suffers the loss $\loss_t(\theta_t) = (1-y_t\theta_t^Tx_t)_+ $,
where $a_+=a$ if $a>0$ and $a_+=0$ otherwise. 

\textbf{Regression} 
At each round $t$, the learner receives a set of features $x_t \in \mathbb{R}^d$ and predicts $y_t\in\mathbb{R}$ using $\left<x_t,\theta_t\right>$. Then the adversary reveals the true value $y_t$ and the learner suffers the loss $\loss_t(\theta_t) = (y_t-f_{\theta_t}(x_t))^2$. We will consider both the linear case when the predictions are linear $f_\theta(x_t)=\theta^Tx_t$ and the nonlinear case where the predictions are outputs of a one-hidden-layer neural network with a ReLU activation. The first case of linear predictions leads to a convex loss with respect to $\theta$, while the latter leads to a nonconvex loss.

\textbf{Variational family} For both tasks, we use a Gaussian mean-field variational family $ \mathcal{F} = \{ q_\mu = \mathcal{N}\left(m,{\rm diag}(\sigma^2)\right) / \mu=(m,\sigma) \in M_m \times M_\sigma \}$, $M_m=[-20,20]^d$ and $M_\sigma=[0,1]^d$.

\textbf{Datasets} We consider here six different datasets: one toy and three real datasets for classification, and one real world dataset for both linear and nonlinear regression. The three real world datasets used for the binary classification problem are the popular Breast Cancer, the Pima Indians and the Forest Cover Type datasets, while those used for regression are the Boston Housing and the California Housing datasets respectively for the convex and the nonconvex case. All come from the UCI machine learning repository. Note that in some databases, the data are ordered according to some criterion such as the date or the label. In order to avoid any effect linked to this, we randomly permuted the observations.

\begin{figure}[!t]
\centering
  \includegraphics[scale=0.35]{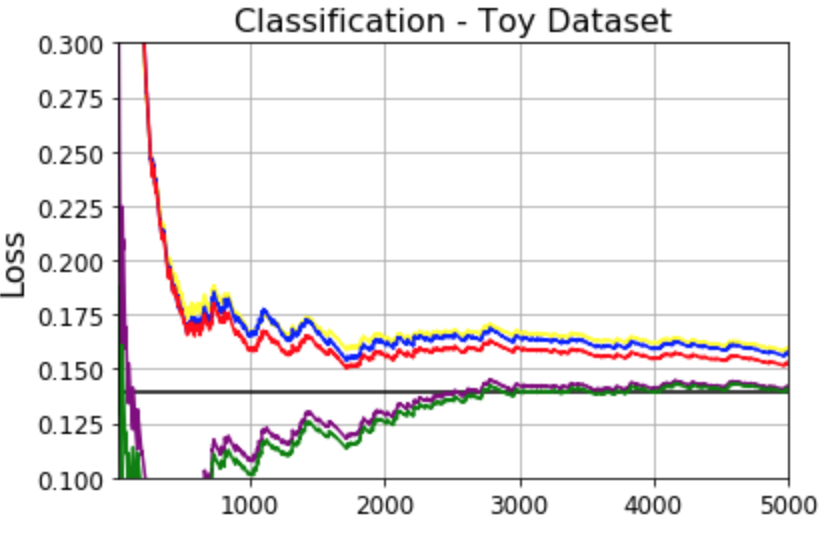}
  \includegraphics[scale=0.35]{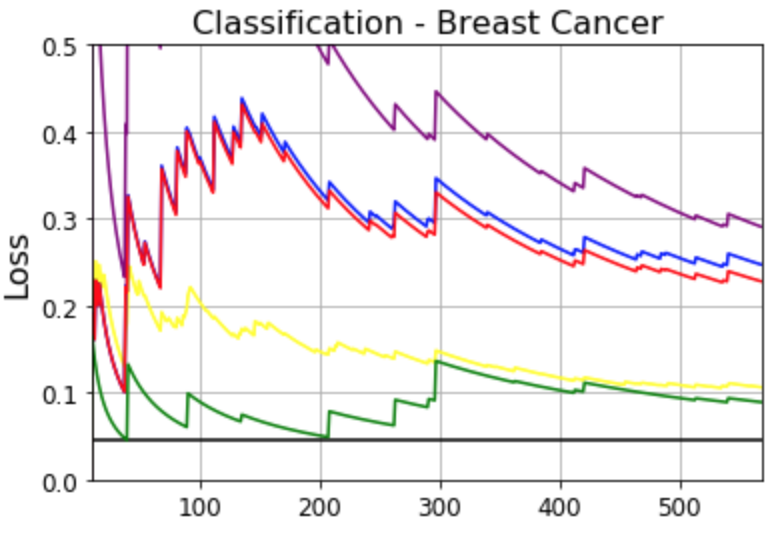}
  \includegraphics[scale=0.35]{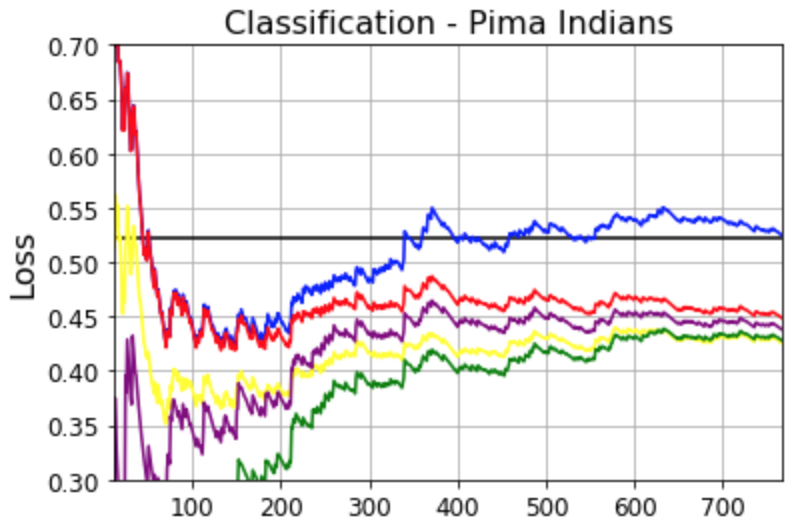}
  \includegraphics[scale=0.349]{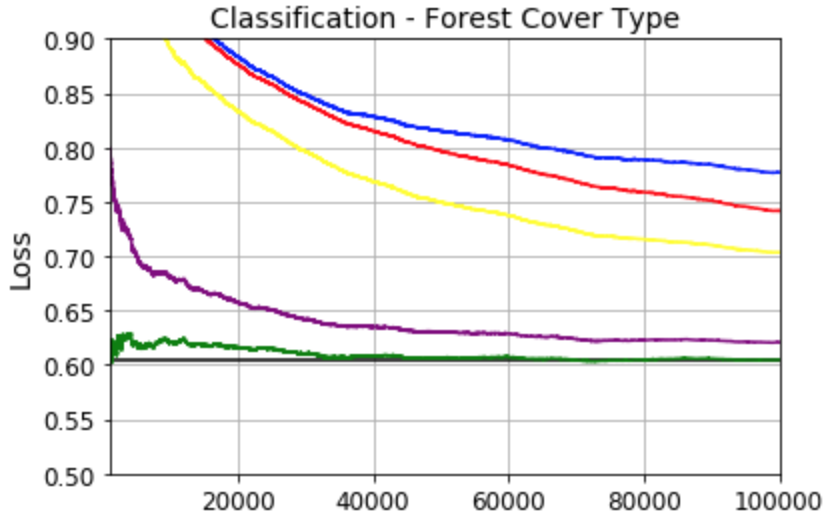}
  \includegraphics[scale=0.349]{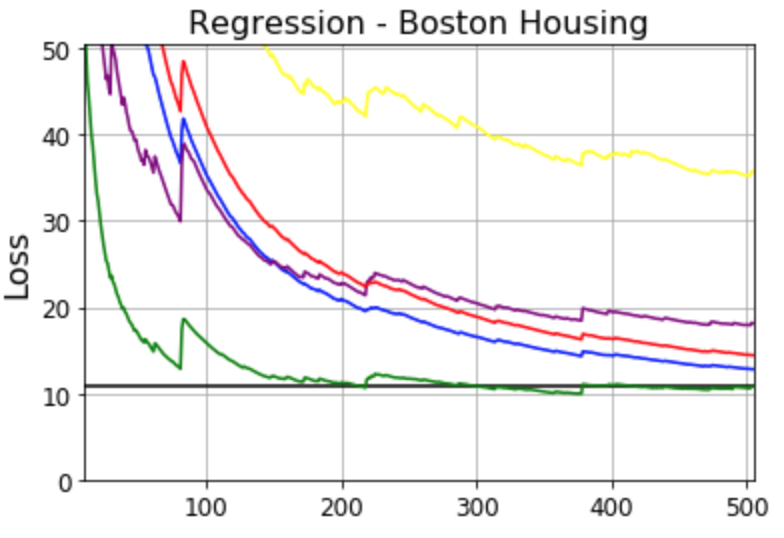}
  \includegraphics[scale=0.349]{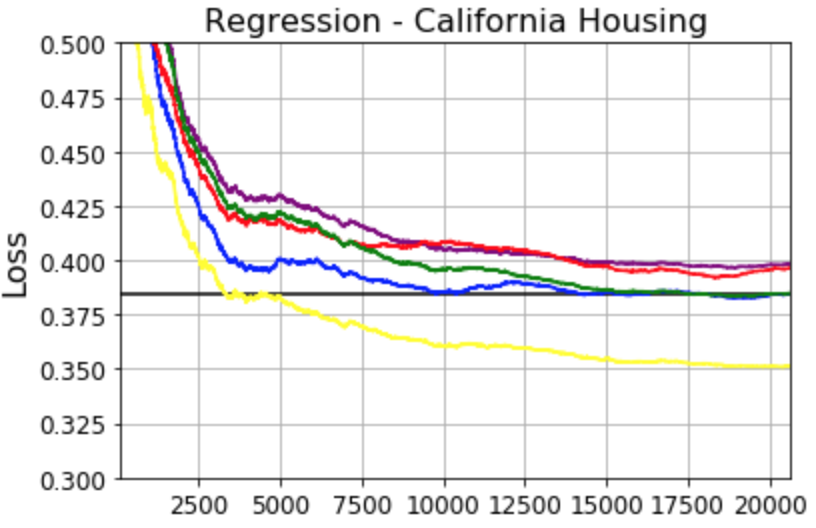}
  \vspace{-0.7cm}
\caption{Average cumulative losses on different datasets for classification and regression tasks with OGA (yellow), OGA-EL (red), SVA (blue), SVB (purple) and NGVI (green) for the convex hinge loss and the squared loss functions. The black line shows the average total cumulative loss in hindsight. We see that in most cases NGVI outperforms the other algorithms. The last plot (California Housing dataset) shows the consistency of our algorithms for a nonconvex loss $\bar{L}_t$.}
\vspace{-0.7cm}
\label{figures}
\end{figure}

\noindent The toy dataset is as follows: we sample $n=10^4$ points $y_t$ according to a Bernoulli distribution $\mathcal{B}e(2/3)$. Then
$$
x_t|(y_t=1) \sim \mathcal{N}
\left(
\left(
\begin{array}{c}
 1
 \\
 1
\end{array}
\right)
,
\left(
\begin{array}{c c}
 1 & 1
 \\
 1 & 3
\end{array}
\right)
\right) \text{ and }
x_t|(y_t=0) \sim \mathcal{N}
\left(
\left(
\begin{array}{c}
 -1
 \\
 -1
\end{array}
\right)
,
\left(
\begin{array}{c c}
 1 & 0
 \\
 0 & 1
\end{array}
\right)
\right).
$$

\begin{center}
\begin{tabular}{|l|c|r||l|c|r|}
  \hline
  Dataset & $T$ & $d$ & Dataset & $T$ & $d$ \\
  \hline
  Toy classification & 10000 & 2 & Cover Type & 581012 & 54 \\
  Breast cancer & 569 & 30 &   Boston Housing & 506 & 13 \\
  Pima Indians & 768 & 8 &   California Housing & 20640 & 9 \\
  \hline
\end{tabular}
\end{center}

\subsection{Experimental results}

For each task and each dataset, we plot the evolution of the average cumulative loss $ \sum_{i=1}^t \loss_i(\theta_i)/t$ as a function of the step $t=1,...,T$, where $T$ is the number of instances of the dataset and $\theta_i$ is the decision made by the learner at step $i$. We compare this quantity to the best average total cumulative loss in hindsight $\inf_{\theta \in M_m}  \frac{1}{T} \sum_{t=1}^T \loss_t(\theta)$
which is represented by a straight black horizontal line in Figure \ref{figures}.

\textbf{Parameters setting} We initialize all means to $0$ and all values of the variance to $1$. For simplicity, the values of the learning rates are set to $\eta=1/\sqrt{T}$ for OGA, OGA-EL and SVA while $\eta_t=1/\sigma_{t}^2\sqrt{t}$ for SVB and $\eta_t = 1$ for NGVI respectively. It is possible to optimize the values of the step sizes. Nevertheless, we draw attention to the fact that a simple cross validation technique would not be valid here as it would require to know the whole dataset before selecting the step size, which is not possible in an online setting, and using such a strategy at each step $t$ using the past data would change the learning rate of OGA, OGA-EL and SVA at each step.

\textbf{Conclusions}
The results are reported in Figure \ref{figures} that shows the consistency of our algorithms. The goal of our simulations is to observe the empirical performance of our algorithms in practice, and to see if it is possible to go further than the convexity assumption that is required in Section \ref{section-bounds}. 
Looking at the plots, the two main findings of our experiments are the following:
\begin{itemize}
\vspace{-0.2cm}
    \item the generalization properties of online variational inference seem to go beyond the convex assumption we stated in the previous theoretical parts.
\vspace{-0.2cm}
    \item even though SVA and SVB exhibit good performances, NGVI is the best method in practice as it converges faster on all the datasets.
\end{itemize}


\section{Conclusion}
In this paper, we derive the first generalization bounds for some online variational inference algorithms. Our proof techniques applies to cases where existing methods do not work.
By using existing variational methods, we proposed a few online methods for variational inference. We provided generalization bounds for the SVA algorithm, and related them to the NGVI methods. We also derived a bound for a special case of SVB. We provided numerical results to establish consistency of our results. We observed that NGVI outperforms all the other methods, and that the theoretical convexity assumption is not needed in practice.

We believe that it is possible to extend our proof techniques to NGVI case. Currently, our proofs strongly rely on the convexity of $\mathbb{E}_{\theta\sim q_\mu}[\ell_t(\theta)]$ with respect to $\mu$. This analysis cannot directly be used for the parameterization of~\cite{khan2017conjugate}. However, it can be applied to a general formulation where our assumptions hold. We believe that generalization bounds for NGVI is possible to derive and will pursue this direction in the future.

\bibliographystyle{apalike}

\begin{thebibliography}{}

\bibitem[Alquier and Ridgway(2017)]{Tempered}
P.~Alquier and J.~Ridgway.
\newblock Concentration of tempered posteriors and of their variational
  approximations.
\newblock \emph{Annals of Statistics (to appear)}, 2017.

\bibitem[Alquier et~al.(2016)Alquier, Ridgway, and
  Chopin]{alquier2016properties}
P.~Alquier, J.~Ridgway, and N.~Chopin.
\newblock On the properties of variational approximations of {G}ibbs
  posteriors.
\newblock \emph{JMLR}, 17\penalty0 (239):\penalty0 1--41, 2016.

\bibitem[Audibert(2009)]{audibert2009fast}
J.~Y. Audibert.
\newblock Fast learning rates in statistical inference through aggregation.
\newblock \emph{Annals of Statistics}, 37\penalty0 (4):\penalty0 1591--1646,
  2009.

\bibitem[Banerjee(2006)]{banerjee2006bayesian}
A.~Banerjee.
\newblock On {B}ayesian bounds.
\newblock In \emph{Proceedings of ICML}, pages 81--88. ACM, 2006.

\bibitem[Baum and Petrie(1966)]{HiddenMarkovModel1966}
L.~E. Baum and T.~Petrie.
\newblock Statistical inference for probabilistic functions of finite state
  {M}arkov chains.
\newblock \emph{Annals of Statistics}, 37\penalty0 (6):\penalty0 1554--1563, 12
  1966.

\bibitem[Bhattacharya et~al.(2016)Bhattacharya, Pati, and
  Yang]{bhattacharya2016bayesian}
A.~Bhattacharya, D.~Pati, and Y.~Yang.
\newblock Bayesian fractional posteriors.
\newblock \emph{arXiv preprint arXiv:1611.01125, to appear in the Annals of
  Statistics}, 2016.

\bibitem[Bhattacharya et~al.(2018)Bhattacharya, Pati, and Yang]{Plage}
A.~Bhattacharya, D.~Pati, and Y.~Yang.
\newblock On statistical optimality of variational {Bayes}.
\newblock \emph{PMLR: Proceedings of AISTAT}, 84, 2018.

\bibitem[Blei and Lafferty(2006)]{Blei2006TopicModels}
D.~M. Blei and J.~D. Lafferty.
\newblock Dynamic topic models.
\newblock In \emph{Proceedings of the 23rd International Conference on Machine
  Learning}, pages 113--120. ACM, 2006.

\bibitem[Broderick et~al.(2013)Broderick, Boyd, Wibisono, Wilson, and
  Jordan]{StreamingVB}
T.~Broderick, N.~Boyd, A.~Wibisono, A.~C. Wilson, and M.~I. Jordan.
\newblock Streaming variational {B}ayes.
\newblock In \emph{NIPS}, pages 1727--1735. Curran Associates, Inc., 2013.

\bibitem[Bubeck(2011)]{bubeck}
S.~Bubeck.
\newblock Introduction to online optimization.
\newblock Lecture notes (Princeton University), 2011.

\bibitem[Catoni(2004)]{catoni2004statistical}
O.~Catoni.
\newblock \emph{Statistical Learning Theory and Stochastic Optimization}.
\newblock Saint-Flour Summer School on Probability Theory 2001, Lecture Notes
  in Mathematics. Springer, 2004.

\bibitem[Catoni(2007)]{MR2483528}
O.~Catoni.
\newblock \emph{{PAC}-{B}ayesian supervised classification: the thermodynamics
  of statistical learning}.
\newblock IMS Lecture Notes, Monograph Series, 56. 2007.

\bibitem[Catoni and Giulini(2017)]{giulini2017}
O.~Catoni and I.~Giulini.
\newblock Dimension free {PAC-B}ayesian bounds for the estimation of the mean
  of a random vector.
\newblock NIPS Workshop: Almost 50 Shades of Bayesian Learning, 2017.

\bibitem[Cesa-Bianchi and Lugosi(2006)]{cesa2006prediction}
N.~Cesa-Bianchi and G.~Lugosi.
\newblock \emph{Prediction, learning, and games}.
\newblock {Cambridge University Press}, 2006.

\bibitem[Challis and Barber(2013)]{challis2013gaussian}
E.~Challis and D.~Barber.
\newblock Gaussian {K}ullback-{L}eibler approximate inference.
\newblock \emph{JMLR}, 14\penalty0 (1):\penalty0 2239--2286, 2013.

\bibitem[Ch{\'e}rief-Abdellatif(2019)]{cherief2018consistency2}
B.-E. Ch{\'e}rief-Abdellatif.
\newblock Consistency of {ELBO} maximization for model selection.
\newblock \emph{PMLR: Proceedings of AABI}, 96:\penalty0 11--31, 2019.

\bibitem[Ch{\'e}rief-Abdellatif and Alquier(2018)]{cherief2018consistency}
B.-E. Ch{\'e}rief-Abdellatif and P.~Alquier.
\newblock Consistency of variational {B}ayes inference for estimation and model
  selection in mixtures.
\newblock \emph{Electronic Journal of Statistics}, 12\penalty0 (2):\penalty0
  2995--3035, 2018.

\bibitem[Cottet and Alquier(2018)]{cottet2018}
V.~Cottet and P.~Alquier.
\newblock 1-bit matrix completion: {PAC-B}ayesian analysis of a variational
  approximation.
\newblock \emph{Machine Learning}, 107\penalty0 (3):\penalty0 579--603, 2018.

\bibitem[Cuong et~al.(2013)Cuong, Ho, and Dinh]{cuong2013generalization}
N.~V. Cuong, L.~S.~T. Ho, and V.~Dinh.
\newblock Generalization and robustness of batched weighted average algorithm
  with {V}-geometrically ergodic {M}arkov data.
\newblock In \emph{International Conference on Algorithmic Learning Theory},
  pages 264--278. Springer, 2013.

\bibitem[Dai et~al.(2016)Dai, He, Dai, and Song]{dai2016provable}
B.~Dai, N.~He, H.~Dai, and L.~Song.
\newblock Provable {B}ayesian inference via particle mirror descent.
\newblock In \emph{AISTAT}, pages 985--994, 2016.

\bibitem[Domke(2019)]{ConvexDomke}
J.~Domke.
\newblock Provable smoothness guarantees for black-box variational inference.
\newblock preprint ArXiv., 2019.

\bibitem[Doucet and Johansen(2009)]{DoucetParticleFiltering}
A.~Doucet and A.~Johansen.
\newblock A tutorial on particle filtering and smoothing: Fifteen years later.
\newblock \emph{Handbook of Nonlinear Filtering}, 12, 01 2009.

\bibitem[Gerchinovitz(2013)]{gerchinovitz2013sparsity}
S.~Gerchinovitz.
\newblock Sparsity regret bounds for individual sequences in online linear
  regression.
\newblock \emph{JMLR}, 14\penalty0 (1):\penalty0 729--769, 2013.

\bibitem[Germain et~al.(2016)Germain, Bach, Lacoste, and
  Lacoste-Julien]{germain2016pac}
P.~Germain, F.~Bach, A.~Lacoste, and S.~Lacoste-Julien.
\newblock Pac-bayesian theory meets bayesian inference.
\newblock In \emph{Advances in Neural Information Processing Systems}, pages
  1884--1892, 2016.

\bibitem[Ghosal and Van~der Vaart(2017)]{ghosal2017fundamentals}
S.~Ghosal and A.~Van~der Vaart.
\newblock \emph{Fundamentals of nonparametric {B}ayesian inference}, volume~44.
\newblock Cambridge {U}niversity {P}ress, 2017.

\bibitem[Gr{\"u}nwald and Van~Ommen(2017)]{grunwaldmisspecifiation}
P.~D. Gr{\"u}nwald and T.~Van~Ommen.
\newblock Inconsistency of {B}ayesian inference for misspecified linear models,
  and a proposal for repairing it.
\newblock \emph{Bayesian Analysis}, 12\penalty0 (4):\penalty0 1069--1103, 2017.

\bibitem[Guedj(2019)]{guedj2019primer}
B.~Guedj.
\newblock A primer on {PAC-B}ayesian learning.
\newblock \emph{Preprint arXiv:1901.05353}, 2019.

\bibitem[Guhaniyogi et~al.(2013)Guhaniyogi, Willett, and Dunson]{dunson2013}
R.~Guhaniyogi, R.~M. Willett, and D.~B. Dunson.
\newblock Approximated {B}ayesian inference for massive streaming data.
\newblock Unpublished manuscript, 2013.

\bibitem[Hazan(2016)]{HazanOnlineConvexOptimization}
E.~Hazan.
\newblock Introduction to online convex optimization.
\newblock \emph{Foundations and Trends{\textregistered} in Optimization},
  2\penalty0 (3--4):\penalty0 157--325, 2016.

\bibitem[Hoffman et~al.(2010)Hoffman, Bach, and Blei]{hoffman2010online}
M.~Hoffman, F.~R. Bach, and D.~M. Blei.
\newblock Online learning for latent dirichlet allocation.
\newblock In \emph{advances in neural information processing systems}, pages
  856--864, 2010.

\bibitem[Hoffman et~al.(2013)Hoffman, Blei, Wang, and
  Paisley]{hoffman2013stochastic}
M.~D. Hoffman, D.~M. Blei, C.~Wang, and J.~Paisley.
\newblock Stochastic variational inference.
\newblock \emph{The Journal of Machine Learning Research}, 14\penalty0
  (1):\penalty0 1303--1347, 2013.

\bibitem[Jaiswal et~al.(2019)Jaiswal, Rao, and Honnappa]{jaiswal2019asymptotic}
P.~Jaiswal, V.~A. Rao, and H.~Honnappa.
\newblock Asymptotic consistency of $\alpha$-r\'enyi-approximate posteriors.
\newblock Preprint arXiv:1902.01902, 2019.

\bibitem[Jordan et~al.(1999)Jordan, Ghahramani, Jaakkola, and
  Saul]{VIJordan1999}
M.~I. Jordan, Z.~Ghahramani, Tommi~S. Jaakkola, and L.~K. Saul.
\newblock An introduction to variational methods for graphical models.
\newblock \emph{Machine Learning}, 37\penalty0 (2):\penalty0 183--233, 1999.

\bibitem[Kalman(1960)]{kalman1960}
R.~E. Kalman.
\newblock A new approach to linear filtering and prediction problems.
\newblock \emph{Transactions of the ASME--Journal of Basic Engineering},
  82\penalty0 (Series D):\penalty0 35--45, 1960.

\bibitem[Khan and Lin(2017)]{khan2017conjugate}
M.~E. Khan and W.~Lin.
\newblock Conjugate-computation variational inference: Converting variational
  inference in non-conjugate models to inferences in conjugate models.
\newblock \emph{PMLR: Proceedings of ICML}, 54:\penalty0 878--887, 2017.

\bibitem[Khan and Nielson(2018)]{khan2018}
M.~E. Khan and D.~Nielson.
\newblock Fast yet simple natural-gradient descent for variational inference in
  complex models.
\newblock Invited paper at {ISITA} 2018, 2018.

\bibitem[Khan et~al.(2018)Khan, Nielsen, Tangkaratt, Lin, Gal, and
  Srivastava]{EmtiForComputations}
M.~E. Khan, D.~Nielsen, V.~Tangkaratt, W.~Lin, Y.~Gal, and A.~Srivastava.
\newblock Fast and scalable {B}ayesian deep learning by weight-perturbation in
  {ADAM}.
\newblock ICML, 2018.

\bibitem[Kingma and Welling(2013)]{kingma2013auto}
Di. Kingma and M.~Welling.
\newblock Auto-encoding variational {B}ayes.
\newblock \emph{Preprint arXiv:1312.6114}, 2013.

\bibitem[Littlestone and Warmuth(1994)]{littlestone1994weighted}
N.~Littlestone and M.~K. Warmuth.
\newblock The weighted majority algorithm.
\newblock \emph{Information and computation}, 108\penalty0 (2):\penalty0
  212--261, 1994.

\bibitem[McAllester(1999)]{mcallester1999some}
D.~A. McAllester.
\newblock Some {PAC-B}ayesian theorems.
\newblock \emph{Machine Learning}, 37\penalty0 (3):\penalty0 355--363, 1999.

\bibitem[Minka(2001)]{MinkaEP}
T.~P. Minka.
\newblock Expectation propagation for approximate {B}ayesian inference.
\newblock In \emph{Proceedings of the 17th Conference in Uncertainty in
  Artificial Intelligence}, pages 362--369. Morgan Kaufmann Publishers Inc.,
  2001.

\bibitem[Nguyen et~al.(2017{\natexlab{a}})Nguyen, Bui, Li, and
  Turner]{nguyen2017online}
C.~V. Nguyen, T.~D. Bui, Y.~Li, and R.~E. Turner.
\newblock Online variational {Bayesian} inference: Algorithms for sparse
  {G}aussian processes and theoretical bounds.
\newblock ICML 2017 Time Series Workshop, 2017{\natexlab{a}}.

\bibitem[Nguyen et~al.(2017{\natexlab{b}})Nguyen, Li, Bui, and
  Turner]{nguyen2017variational}
C.~V. Nguyen, Y.~Li, T.~D. Bui, and R.~E. Turner.
\newblock Variational continual learning.
\newblock \emph{Preprint arXiv:1710.10628}, 2017{\natexlab{b}}.

\bibitem[Rousseau(2016)]{rousseau2016frequentist}
J.~Rousseau.
\newblock On the frequentist properties of {B}ayesian nonparametric methods.
\newblock \emph{Annual Review of Statistics and Its Application}, 3:\penalty0
  211--231, 2016.

\bibitem[Sato(2001)]{sato2001online}
M.-A. Sato.
\newblock Online model selection based on the variational {B}ayes.
\newblock \emph{Neural computation}, 13\penalty0 (7):\penalty0 1649--1681,
  2001.

\bibitem[Seldin and Tishby(2010)]{seldin2010pac}
Y.~Seldin and N.~Tishby.
\newblock {PAC}-{B}ayesian analysis of co-clustering and beyond.
\newblock \emph{JMLR}, 11:\penalty0 3595--3646, 2010.

\bibitem[Seldin et~al.(2011)Seldin, Auer, Shawe-Taylor, Ortner, and
  Laviolette]{seldin2011pac}
Y.~Seldin, P.~Auer, J.~Shawe-Taylor, R.~Ortner, and F.~Laviolette.
\newblock Pac-bayesian analysis of contextual bandits.
\newblock In \emph{Advances in Neural Information Processing Systems}, pages
  1683--1691, 2011.

\bibitem[Shalev-Shwartz(2012)]{shalev2012online}
S.~Shalev-Shwartz.
\newblock Online learning and online convex optimization.
\newblock \emph{Foundations and Trends{\textregistered} in Machine Learning},
  4\penalty0 (2):\penalty0 107--194, 2012.

\bibitem[Shawe-Taylor and Williamson(1997)]{shawe1997pac}
J.~Shawe-Taylor and R.~C. Williamson.
\newblock A {PAC} analysis of a {B}ayesian estimator.
\newblock In \emph{Tenth annual conference on Computational learning theory},
  volume~6, pages 2--9, 1997.

\bibitem[Sheth and Khardon(2017)]{sheth2017excess}
R.~Sheth and R.~Khardon.
\newblock Excess risk bounds for the {B}ayes risk using variational inference
  in latent {G}aussian models.
\newblock In \emph{NIPS}, pages 5151--5161, 2017.

\bibitem[Suzuki(2012)]{suzuki2012pac}
T.~Suzuki.
\newblock {PAC-B}ayesian bound for {G}aussian process regression and multiple
  kernel additive model.
\newblock In \emph{Conference on Learning Theory}, pages 8--1, 2012.

\bibitem[Tsuzuku et~al.(2019)Tsuzuku, Sato, and
  Sugiyama]{tsuzuku2019normalized}
Y.~Tsuzuku, I.~Sato, and M.~Sugiyama.
\newblock Normalized flat minima: Exploring scale invariant definition of flat
  minima for neural networks using {PAC-B}ayesian analysis.
\newblock \emph{Preprint arXiv:1901.04653}, 2019.

\bibitem[Vovk(1990)]{Vovk:1990:AS:92571.92672}
V.~G. Vovk.
\newblock Aggregating strategies.
\newblock In \emph{Proceedings of the Third Annual Workshop on Computational
  Learning Theory}, 1990.

\bibitem[Wang et~al.(2011)Wang, Paisley, and Blei]{wang2011online}
C.~Wang, J.~Paisley, and D.~Blei.
\newblock Online variational inference for the hierarchical {D}irichlet
  process.
\newblock In \emph{Proceedings of AISTAT 2011}, pages 752--760, 2011.

\bibitem[Wang and Blei(2018)]{wang2018frequentist}
Y.~Wang and D.~M. Blei.
\newblock Frequentist consistency of variational {B}ayes.
\newblock Journal of the American Statistical Association (to appear), 2018.

\bibitem[Zeno et~al.(2018)Zeno, Golan, Hoffer, and Soudry]{zeno2018}
C.~Zeno, I.~Golan, E.~Hoffer, and D.~Soudry.
\newblock {B}ayesian gradient descent: Online variational {B}ayes learning with
  increased robustness to catastrophic forgetting and weight pruning.
\newblock Preprint arXiv:1803.10123, 2018.

\bibitem[Zhang and Gao(2017)]{Chicago}
F.~Zhang and C.~Gao.
\newblock Convergence rates of variational posterior distributions.
\newblock \emph{Preprint arXiv:1712.02519v1}, 2017.

\end{thebibliography}

\newpage

\section*{Appendix}

\subsection{Closed-form solutions for NGVI}

The expectation parameterization of NGVI enables closed-form solution. This is because the gradient of the KL diverence with respect to expectation parameter is available in closed-form (see Eq. 10 in \cite{khan2018}). The closed update is given in Eq. (50) in \cite{khan2017conjugate} using which we obtain the following update:
\begin{align}
\lambda_{t+1} = (1-\beta)\lambda_t + \beta \lambda_1 - \eta\beta \nabla_\mu \bar{L}_t(\mu_t),
\label{eq:ngvi_recursive}
\end{align}
where $1/\beta := 1/\alpha + 1/\eta$.
Given $\lambda_{t+1}$, we can get $\mu_{t+1} = \nabla_\lambda A(\lambda_{t+1})$ where $A$ is the log-partition function of the exponential family.

Now we show that this closed-form update is similar to SVA. By using induction similar to Lemma 4~in \cite{khan2017conjugate}, we can write the update in terms of all past gradients:
\begin{align}
\lambda_{t+1} = \lambda_1 - \eta \sum_{i=1}^t w_i \nabla_\mu \bar{L}_i(\mu_i) 
\end{align}
where $w_i := \beta \prod (1-\beta)^{i-2}$. This can be compared to the SVA update in the expectation parameterization where applying the gradient to \eqref{eq:sva_opt} gives us the following update similar to \eqref{eq:ngvi_recursive} but where $w_i =1$ for all $i$:
\begin{align}
\lambda_{t+1} = \lambda_1 - \eta \sum_{i=1}^t   \nabla_\mu \bar{L}_i(\mu_i)
\end{align}
Therefore, SVA takes a gradient step assuming that all gradients are equally important, which is similar to the Bayesian update \eqref{eq:temperedBayes} where all loss $\loss_i$ are treated equally. In contrast, in NGVI, the past gradients are discounted using $\beta$ and ultimately forgotten. Weighting past gradients makes sense when we do not want the current mistakes to affect the future. However, the choice of step-size is crucial to know the rate at which the past gradients should be discounted.

NGVI is typically applied using expectation parameterization, but the formulation \eqref{eq:ngvi} is more general although could be computationally difficult. The theoretical results in the paper further assume that $\bar{L}_i$ is convex in $\mu$. Still, in our experiments, NGVI gives good performance in an online setting compared to many other algorithms.

\subsection{Online gradient algorithm on the expected loss (OGA-EL)}
\label{Appendix:OGA-EL}

It is possible to directly use the online gradient algorithm (OGA) on the expected loss $\mathbb{E}_{\theta\sim q_\mu}[\loss_t(\theta)]$, see Algorithm~\ref{algo-onlineVB2}.

\begin{algorithm}
\caption{OGA-EL}
\label{algo-onlineVB2}
\begin{description}
\item[Input] Learning rate $\eta>0$, a prior $\pi(\theta) \in \mathcal{F}$, $q_{\mu_1} \leftarrow \pi$.
\item[Loop] For $t=1,\dots$,
\begin{description}
\item[1.] $\hat{\theta}_{t}\leftarrow\mathbb{E}_{\theta\sim q_{\mu_t}}[\theta]$,
\item[2.] Observe $\data_t$ to suffer a loss $\loss_t(\hat{\theta}_t)$.
\item[3.] Update $ \mu_{t+1} = \mu_{t} -\eta \nabla \bar{L}_t(\mu_t)  .$
\end{description}
\end{description}
\end{algorithm}
Note first that from~\cite{shalev2012online} step (iii) is actually equivalent to
$$ \mu_{t+1} 
= \argmin_{\mu\in M} \Biggl[ \sum_{i=1}^{t} \mu^T \nabla \bar{L}_i(\mu_i)
+ \frac{\|\mu-\mu_1\|^2}{\eta} \Biggr] ,$$
which means that we replaced the K\"ullback-Leibler divergence by the Euclidean norm in SVA.

Also, when $\mu=(m,\sigma)\in\mathbb{R}^d \times (\mathbb{R}_+)^d $ and $q_\mu = \mathcal{N}(m,{\rm diag}(\sigma))$, then Algorithm \ref{algo-onlineVB2} becomes
\begin{align*}
& m_{t+1} = m_{t} - \eta s^2 \frac{\partial \bar{L}_t}{\partial m}(m_t,\sigma_t) , \\
& \sigma_{t+1} = \sigma_t - \eta s^2 \frac{\partial \bar{L}_t}{\partial \sigma}(m_t,\sigma_t).
\end{align*}

We have regret bounds for this method, similar to the one for EWA:
\begin{thm}
\label{thm2}
Under Assumption~\ref{asm-l}, Algorithm~\ref{algo-onlineVB2} leads to:
\begin{equation*}
\sum_{t=1}^T \loss_{t}(\hat{\theta}_t)
\leq \inf_{\mu\in M} \Biggl\{\mathbb{E}_{\theta\sim q_\mu}\left[ \sum_{t=1}^T  \loss_t(\theta) \right] + \eta L^2 T + \frac{\|\mu-\mu_1\|^2}{\eta} \Biggr\},
\end{equation*}
and moreover, under Assumptions~\ref{asm-k} and~\ref{asm-l}, Algorithm~\ref{algo-onlineVB2} leads to:
\begin{equation*}
\sum_{t=1}^T \loss_{t}(\hat{\theta}_t)
\leq \inf_{\mu\in M} \Biggl\{\mathbb{E}_{\theta\sim q_\mu}\left[ \sum_{t=1}^T  \loss_t(\theta) \right] + \eta L^2 T + \frac{\alpha \mathcal{K}(q_\mu,\pi)}{2 \eta} \Biggr\}.
\end{equation*}
\end{thm}
The proof of this result is given below with the other proofs of the paper.

\subsection{Online-to-batch conversion}
\label{subsection-otb}

Many times in the paper, we derived generalization error bounds from regret bounds, using the online-to-batch conversion. We here give a formal statement for this result, note that this result is essentially Theorem 5.1 in~\cite{shalev2012online}. We also provide a proof for the sake of completeness.

\begin{thm}
\label{ssthm5}
 Assume that $\data_1,\dots,\data_T$ are i.i.d from $P_*$. Assume we use an online algorithm on the data that produce a sequence of parameters $\hat{\theta}_1,\dots,\hat{\theta}_T$. That is, $\hat{\theta}_t=\hat{\theta}(\data_1,\dots,\data_{t-1})$. Define the estimator
 $$
\bar{\theta}_T = \frac{1}{T}\sum_{t=1}^T \hat{\theta}_t.
 $$
 Then
$$ \myexpect_{\mathcal{D}_{1:T}\sim P_*} [ \generr_*(\bar{\theta}_T)] \leq  \myexpect_{\mathcal{D}_{1:T}\sim P_*}\left[ \frac{1}{T} \sum_{t=1}^T \loss_t(\hat{\theta}_t) \right] . $$
\end{thm}

\begin{proof}
We have:
$$
  \mathcal{E}_*(\bar{\theta}_T) = \mathbb{E}_{\data\sim P_*}\left[ \loss(\data,\bar{\theta}_T) \right]
   = \mathbb{E}_{\data\sim P_*}\left[ \loss\left(\data,\frac{1}{T}\sum_{t=1}^T \hat{\theta}_t \right) \right]  \leq \frac{1}{T}\sum_{t=1}^T \mathbb{E}_{\data\sim P_*}\left[ \loss\left(\data, \hat{\theta}_t \right) \right] 
$$
by Jensen's inequality. The key is that as $\hat{\theta}_t = \hat{\theta}_t(\data_1,\dots,\data_{t-1})$ does not depend on $\data_t$, we can rewrite:
$$
\mathbb{E}_{\data\sim P_*}\left[ \loss\left(\data, \hat{\theta}_t \right) \right]
= \mathbb{E}_{\data_t \sim P_*}\left[ \loss\left(\data_t, \hat{\theta}_t \right) \right]
=  \mathbb{E}_{\data_t \sim P_*}\left[ \loss_t(\hat{\theta}_t) \right]
$$
and so we have
 \begin{align*}
   \myexpect_{\mathcal{D}_{1:T}\sim P_*}\left[
  \mathcal{E}_*(\bar{\theta}_T)
  \right]
& \leq   \myexpect_{\mathcal{D}_{1:T}\sim P_*}\left\{ \frac{1}{T}\sum_{t=1}^T \mathbb{E}_{\data\sim P_*}\left[  \loss_t(\hat{\theta}_t) \right] \right\}
\\
& = \frac{1}{T}\sum_{t=1}^T  \myexpect_{\mathcal{D}_{1:T}\sim P_*}\left[  \loss_t(\hat{\theta}_t) \right]
 = \myexpect_{\mathcal{D}_{1:T}\sim P_*}\left[   \frac{1}{T}\sum_{t=1}^T  \loss_t(\hat{\theta}_t) \right].
 \end{align*}
\end{proof}

As an application, we state an exact version of~\eqref{eq-generr-EWA} and prove it from Theorem~\ref{thm-EWA} and Theorem~\ref{ssthm5}.

\begin{thm}
\label{thm-dim}
Assume that the loss $\loss$ is bounded by $B$ as in Theorem~\ref{thm-EWA} and that $\data_1,\dots,\data_T$ are i.i.d from $P_*$. Assume that there is some $d>0$ such that
$$ r(\varepsilon) \leq -d\log(1/\varepsilon)$$
where $r(\varepsilon) = \log[1/ \pi(B(\theta^*,\varepsilon))]$ and $
B(\theta^*,\varepsilon) = \{\theta\in \Theta: \mathcal{E}(\theta) - \mathcal{E}(\theta^*) \leq \varepsilon \}$.
Use on this data the EWA strategy with $\eta =(1/2\sqrt{2} B)\sqrt{(d/T)\log(d/T)}$, then
\begin{equation*}
\myexpect_{\mathcal{D}_{1:T}\sim P_*}[ \mathcal{E}_*(\hat{\theta}_T)]
\leq  \mathcal{E}_*(\theta^*) + B \sqrt{\frac{d}{2T}\log\left(\frac{T}{d}\right)} + \frac{d}{T}.
\end{equation*}
\end{thm}
Note that the prior mass condition is classical in the PAC-Bayesian literature and in the frequentist analysis of Bayesian estimators, see e.g~~\cite{MR2483528,rousseau2016frequentist,bhattacharya2016bayesian,ghosal2017fundamentals}. The estimator $\bar{\theta}_T$ averaging the decisions $\hat{\theta}_t$ was first introduced by~\cite{catoni2004statistical} as the "double mixture rule". 
\begin{proof}
Define $p_\varepsilon$ as $\pi$ restricted to $B(\theta^*,\varepsilon)$ and note that
$$ \mathcal{K}(p_\varepsilon,\pi) = -\log B(\theta^*,\varepsilon) = r(\varepsilon) \leq d \log(1/\varepsilon) .$$
From Theorem~\ref{thm-EWA}, for any $\varepsilon$,
\begin{align*}
\sum_{t=1}^T \loss_t(\hat{\theta}_t) \leq
 \mathbb{E}_{\theta\sim
 p_\varepsilon}\left[\sum_{t=1}^T \loss_t(\theta)\right] + \frac{\eta B^2 T}{8} + \frac{d \log(1/\varepsilon)}{\eta}.
\end{align*}
From Theorem~\ref{ssthm5},
\begin{align*}
\myexpect_{\mathcal{D}_{1:T}\sim P_*} [ \mathcal{E}_*(\hat{\theta}_T)]
& = \myexpect_{\mathcal{D}_{1:T}\sim P_*} \left[ \frac{1}{T} \sum_{t=1}^T \loss_t(\hat{\theta}_t) \right]
\\
& \leq
\myexpect_{\mathcal{D}_{1:T}\sim P_*}\left\{
 \mathbb{E}_{\theta\sim
 p_\varepsilon}\left[\frac{1}{T} \sum_{t=1}^T \loss_t(\theta)\right]\right\}
  + \frac{\eta B^2}{8} + \frac{d \log(1/\varepsilon)}{T \eta}
 \\
& =  \mathbb{E}_{\theta\sim
 p_\varepsilon}\left[\generr_*(\theta)\right] + \frac{\eta B^2}{8} + \frac{d \log(1/\varepsilon)}{T \eta}
 \\
& \leq \generr_*(\theta^*) + \varepsilon + \frac{\eta B^2}{8} + \frac{d \log(1/\varepsilon)}{T \eta}
\end{align*}
where the last inequality comes from the definition of $p_\varepsilon$. Taking $\varepsilon = d/T$ gives:
\begin{equation*}
\myexpect_{\mathcal{D}_{1:T}\sim P_*} [ \mathcal{E}_*(\hat{\theta}_T)] \leq
\generr_*(\theta^*) + \frac{d}{T} + \frac{\eta B^2}{8} + \frac{d \log(T/d)}{T \eta}.
\end{equation*}
Finally, substitute its value to $\eta$ to get
\begin{equation*}
\myexpect_{\mathcal{D}_{1:T}\sim P_*} [ \mathcal{E}_*(\hat{\theta}_T)] \leq
\generr_*(\theta^*) + B \sqrt{\frac{d}{2T} \log\left(\frac{T}{d}\right)} + \frac{d}{T}.
\end{equation*}
\end{proof}

\subsection{A tool for the proofs}

We remind the following classical lemma. We refer the reader for example to~\cite{MR2483528} for a proof of this result, where it is stated as Lemma 1.1.3 (page 16).
\begin{lemma}
\label{DV}
 Let $h:\Theta\rightarrow \mathbb{R}$ be a bounded measurable function and $\pi\in\mathcal{S}(\Theta)$. Then
 $$ \sup_{p\in\mathcal{S}(\Theta)}\left\{ \mathbb{E}_{\theta\sim p}[h(\theta)] - \mathcal{K}(p,\pi) \right\} = \log \mathbb{E}_{\theta\sim\pi} [\exp(h(\theta))] $$
 and the supremum is actually reached for
 $$ p(\theta) \propto \exp[h(\theta)]\pi(\theta). $$
\end{lemma}
This lemma will actually turn out to be a fundamental tool for some of the proofs.

\subsection{Proofs}
\label{proofs}

\begin{proof}[Proof of Theorem~\ref{thm-EWA}]
Note that this proof is classical and is reminded here for the sake of completeness.
We have first:
\begin{align*}
 \exp\left[ -\eta \loss_t(\hat{\theta}_t) \right]
 & = \exp\left[ -\eta \loss_t(\mathbb{E}_{\theta\sim p_t^\eta}(\theta)) \right]
 \\
 & \geq  \exp\left[ -\eta \mathbb{E}_{\theta\sim p_{t}^\eta}(\loss_t(\theta)) \right]
 \\
 & \geq  \mathbb{E}_{\theta\sim p_t^\eta}\left\{ \exp\left[ -\eta \loss_t(\theta) - \frac{\eta^2 B^2}{8}  \right]\right\}
\end{align*}
where we used respectively Jensen and Hoeffding's inequality. So
\begin{equation}
\label{step1}
\loss_t(\hat{\theta}_t) \leq \frac{\eta B^2}{8} - \frac{1}{\eta} \log  \mathbb{E}_{\theta\sim p_t^\eta} \exp\left[ -\eta \loss_t(\theta) \right].
\end{equation}
Remind that by definition,
$$ p_t^\eta(\theta) = \frac{\exp\left(-\eta \sum_{i=1}^{t-1} \loss_i(\theta) \right) \pi(\theta) }{N_t} $$
where $N_t$ is the normalisation constant given by
$$ N_t = \mathbb{E}_{\theta\sim \pi} \left[\exp\left(-\eta \sum_{i=1}^{t-1} \loss_i(\theta) \right)\right]. $$
But note that then
$$ \log \mathbb{E}_{\theta\sim p_t^\eta} \exp\left[ -\eta \loss_t(\theta) \right] = \log\left(\frac{N_{t+1}}{N_t} \right) .$$
We plug this into~\eqref{step1} and sum for $t=1,\dots,T$. We obtain
\begin{align*}
\sum_{t=1}^T \loss_t(\hat{\theta}_t)
& \leq \frac{\eta B^2 T}{8} - \frac{1}{\eta} \sum_{t=1}^T \log\left(\frac{N_{t+1}}{N_t} \right)
\\
& = \frac{\eta B^2 T}{8} - \frac{1}{\eta} \log\left(\frac{N_{T+1}}{N_1} \right)
\\
& =  \frac{\eta B^2 T}{8} - \frac{1}{\eta} \log\left( \mathbb{E}_{\theta\sim \pi} \left[\exp\left(-\eta \sum_{t=1}^{T} \loss_t(\theta) \right)\right] \right).
\end{align*}
Lemma~\ref{DV} leads to
\begin{equation*}
\sum_{t=1}^T \loss_t(\hat{\theta}_t)
\leq \frac{\eta B^2 T}{8}  +  \inf_{p\in\mathcal{S}(\Theta)} \left\{ \mathbb{E}_{\theta\sim p}\left[\sum_{t=1}^T \loss_t(\theta)\right] + \frac{\mathcal{K}(p,\pi)}{\eta} \right\}.
\end{equation*}
\end{proof}

\begin{proof}[Proof of Proposition~\ref{prop-lipschitz}]
Let $\varphi_{m,C}(\cdot)$ denote the p.d.f of the Gaussian distribution with mean $m$ and variance matrix $C$.
Let $(m_1,C_1),(m_2,C_2)\in M$,
\begin{align*}
 | \bar{L}_t(m_1,C_1) - \bar{L}_t(m_2,C_2) |
 & = \left| \int \loss_t(\theta) \varphi_{m_1,C_1}(\theta) {\rm d}\theta
   -  \int \loss_t(\theta) \varphi_{m_2,C_2}(\theta) {\rm d}\theta \right|
 \\
 & \leq \int \left| \loss_t(m_1+C_1 u) - \loss_t(m_2 + C_2 u) \right| \varphi_{0,I_d}(u) {\rm d}u
 \\
 & \leq L'\|m_1-m_2\| +L'\int \|(C_1-C_2)u\| \varphi_{0,I_d}(u) {\rm d}u.
\end{align*}
For any $C=(C_{i,j})\in UT(d)$, we have
\begin{align*}
 \int \|Cu\| \varphi_{0,I_d}(u) {\rm d}u
 & \leq \sqrt{ \int \|Cu\|^2 \varphi_{0,I_d}(u) {\rm d}u}
 \\
 & =  \sqrt{ \int \sum_{i=1}^d \left( \sum_{j=1}^d C_{i,j} u_j \right)^2 \varphi_{0,I_d}(u) {\rm d}u}
 = \sqrt{ \sum_{i=1}^d \sum_{j=1}^d C_{i,j}^2 }
\end{align*}
which leads to
\begin{align*}
| \bar{L}_t(m_1,C_1) - \bar{L}_t(m_2,C_2) |
& \leq L'\|m_1-m_2\| +L'\sqrt{ \sum_{i=1}^d \sum_{j=1}^d (C_1-C_2)_{i,j}^2 }
\\
& \leq 2 L' \|(m_1,C_1)-(m_2,C_2)\|.
\end{align*}
This ends the proof.
\end{proof}

\begin{proof}[Proof of Theorem~\ref{thm1}]
First, Assumption~\ref{asm-l} ensures that the $\bar{L}_t$'s are convex. By definition of the subgradient of a convex function,
\begin{align}
\sum_{t=1}^{T}  \loss_t(\hat{\theta}_t)
-  \sum_{t=1}^{T}\mathbb{E}_{\theta\sim q_\mu} [\loss_t(\theta)] 
\nonumber
& = \sum_{t=1}^{T}  \loss_t\left(\mathbb{E}_{\theta\sim q_{\mu_t}}(\theta) \right)
-  \sum_{t=1}^{T} \mathbb{E}_{\theta\sim q_\mu} [\loss_t(\theta)] 
\nonumber
\\
& \leq  
\sum_{t=1}^{T} \mathbb{E}_{\theta\sim q_{\mu_t}}[ \loss_t(\theta) ]
-  \sum_{t=1}^{T} \mathbb{E}_{\theta\sim q_\mu} [\loss_t(\theta)] 
\nonumber
\\
& = \sum_{t=1}^T \bar{L}_t(\mu_t) - \sum_{t=1}^t \bar{L}_t(\mu)
\nonumber
\\
& \leq \sum_{t=1}^T \mu_t^T \nabla \bar{L}_t(\mu_t) - \sum_{t=1}^T \mu^T \nabla \bar{L}_t(\mu_t).
\label{eq-step1}
\end{align}

Then, following the general proof scheme detailed in Chapter 2 in~\cite{shalev2012online}, we prove by recursion on $T$ that for any $\mu\in\mathcal{M}$,
\begin{equation}
\sum_{t=1}^T \mu_t^T \nabla \bar{L}_t(\mu_t) - \sum_{t=1}^T \mu^T \nabla \bar{L}_t(\mu_t)
\leq
\sum_{t=1}^T \mu_t^T \nabla \bar{L}_t(\mu_t) - \sum_{t=1}^T \mu_{t+1}^T \nabla \bar{L}_t(\mu_t) + \frac{\mathcal{K}(q_\mu,\pi)}{\eta}
\label{eq-lemma23-shalev}
\end{equation}
which is exactly equivalent to
\begin{equation}
 \label{eq-lemma23-shalev-prime}
\sum_{t=1}^T \mu_{t+1}^T \nabla \bar{L}_t(\mu_t)
 \leq  \sum_{t=1}^T \mu^T \nabla \bar{L}_t(\mu_t) + \frac{\mathcal{K}(q_\mu,\pi)}{\eta}.
\end{equation}
Indeed, for $T=0$,~\eqref{eq-lemma23-shalev-prime} just states that $\mathcal{K}(q_\mu,\pi) \geq 0$ which is a well-known property of KL. Assume that~\eqref{eq-lemma23-shalev-prime} holds for some integer $T-1$.
We then have, for all $\mu\in M$,
\begin{align*}
\sum_{t=1}^{T} \mu_{t+1}^T \nabla \bar{L}_t(\mu_t)
 & = \sum_{t=1}^{T-1} \mu_{t+1}^T \nabla \bar{L}_t(\mu_t) + \mu_{T+1}^T \nabla \bar{L}_T(\mu_T)
 \\
 & \leq  \sum_{t=1}^{T-1} \mu^T \nabla \bar{L}_t(\mu_t) + \frac{\mathcal{K}(q_\mu,\pi)}{\eta} + \mu_{T+1}^T \nabla \bar{L}_T(\mu_T)
\end{align*}
as~\eqref{eq-lemma23-shalev-prime} holds for $T-1$. Apply this to $\mu=\mu_{T+1}$ to get
\begin{align*}
\sum_{t=1}^{T} \mu_{t+1}^T \nabla \bar{L}_t(\mu_t)
& \leq \sum_{t=1}^{T} \mu_{T+1}^T \nabla \bar{L}_t(\mu_t) + \frac{\mathcal{K}(p_{\mu_{T+1}},\pi)}{\eta}
\\
& = \min_{m\in \mathcal{M}} \left[\sum_{t=1}^T m^T \nabla \bar{L}_t(\mu_t) + \frac{\mathcal{K}(p_m,\pi)}{\eta}\right]\text{, by definition of } \mu_{T+1}
\\
& \leq \sum_{t=1}^T \mu^T \nabla \bar{L}_t(\mu_t) + \frac{\mathcal{K}(q_\mu,\pi)}{\eta}
\end{align*}
for all $\mu\in \mathcal{M}$. Thus,~\eqref{eq-lemma23-shalev-prime} holds for $T$. Thus, by recursion,~\eqref{eq-lemma23-shalev-prime} and~\eqref{eq-lemma23-shalev} hold for all $T\in\mathbb{N}$.

The last step is to prove that for any $t\in\mathbb{N}$,
\begin{equation}
 \label{eq-FTRL}
  \mu_t^T \nabla \bar{L}_t(\mu_t) - \mu_{t+1}^T \nabla \bar{L}_t(\mu_t)
 \leq \frac{\eta L^2}{\alpha}.
\end{equation}
Indeed,
\begin{align}
 \mu_t^T  \nabla \bar{L}_t(\mu_t) - \mu_{t+1}^T \nabla \bar{L}_t(\mu_t)
 & = (\mu_t - \mu_{t+1})^T \nabla \bar{L}_t(\mu_t)
  \nonumber
 \\
 & \leq \|\mu_t-\mu_{t+1}\| \|\nabla \bar{L}_t(\mu_t)\|
  \text{ by Cauchy-Schwarz}
  \nonumber
 \\
 & \leq L \|\mu_t - \mu_{t+1}\|
 \label{eq-FTRL-step1}
\end{align}
as $\bar{L}_t$ is $L$ Lipschitz (Assumption~\ref{asm-l}). Define
$$
G_{t}(\mu) = \sum_{i=1}^{t-1} \mu^T \nabla \bar{L}_i(\mu_i) + \frac{\mathcal{K}(q_\mu,\pi)}{\eta}.
$$
Note that from Assumption~\ref{asm-k}, $\mu\mapsto \mathcal{K}(q_\mu,\pi)/\eta$ is $\alpha/\eta$-strongly convex. As the sum of a linear function and an $\alpha/\eta$-strongly convex function, $G_t$ is $\alpha/\eta$-strongly convex. So, for any $(\mu,\mu')$,
$$
G_{t}(\mu') - G_{t}(\mu) \geq (\mu'-\mu)^T \nabla G_t(\mu) + \frac{\alpha \|\mu'-\mu\|^2}{2 \eta} .
$$
As a special case, using the fact that $\mu_t$ is a minimizer of $G_t$, we have
$$
G_{t}(\mu_{t+1}) - G_{t}(\mu_t) \geq \frac{\alpha \|\mu_{t+1}-\mu_t\|^2}{2 \eta}.
$$
In the same way,
$$
G_{t+1}(\mu_{t}) - G_{t+1}(\mu_{t+1}) \geq \frac{\alpha \|\mu_{t+1}-\mu_t\|^2}{2 \eta}.
$$
Summing the two previous inequalities gives
$$
\mu_t^T  \nabla \bar{L}_t(\mu_t) - \mu_{t+1}^T \nabla \bar{L}_t(\mu_t) \geq \frac{\alpha \|\mu_{t+1}-\mu_t\|^2}{\eta},
$$
and so, combined with, this gives:
\begin{align*}
 \|\mu_{t+1}-\mu_t\|
  \leq \sqrt{\frac{\eta}{\alpha}\left[ \mu_t^T \nabla \bar{L}_t(\mu_t) - \mu_{t+1}^T \nabla \bar{L}_t(\mu_t)\right]}.
\end{align*}
Combining this inequality with~\eqref{eq-FTRL-step1} leads to~\eqref{eq-FTRL}.

Plugging~\eqref{eq-step1},~\eqref{eq-lemma23-shalev} and~\eqref{eq-FTRL} together gives
\begin{equation*}
\sum_{t=1}^{T}  \loss_t(\hat{\theta}_t)
-  \sum_{t=1}^{T} \mathbb{E}_{\theta\sim q_\mu} [\loss_t(\theta)]
\leq \frac{\eta T L^2}{\alpha} + \frac{\mathcal{K}(q_\mu,\pi)}{\eta},
\end{equation*}
that is the statement of the theorem.
\end{proof}

\begin{proof}[Proof of Theorem~\ref{thm3}]
We prove this theorem from scratch and use the main techniques outlined in \cite{HazanOnlineConvexOptimization}. As previously, the idea is to study differences $\bar{L}_t(\mu_t) -\bar{L}_t(\mu)$. However, in this case, we have, for any $\mu=(m,\sigma)$, using Jensen's inequality,
$$ \bar{L}_t(m,\sigma) = \mathbb{E}_{\theta\sim q_{m,\sigma}}[\ell_t(\theta)] \geq \ell_t(m) = \bar{L}_t(m,0). $$
So, we can assume from the beginning that $\mu=(m,0)$.

\subsection*{Convex case:}
First, we assume that each function $\bar{L}_t$ is convex, for all $m=(m_1,...,m_d)\in \mathcal{M}_m$ and $\mu=(m,0)$:
$$
    \bar{L}_t (\mu_t) -\bar{L}_t(\mu)  \leq \nabla \bar{L}_t(\mu_t) ^T (\mu_t-\mu)  = \sum_{j=1}^{d} \left[ \frac{\partial \bar{L}_t}{\partial m_j}(m_t,\sigma_t)(m_{t,j}-m_j) + \frac{\partial \bar{L}_t}{\partial \sigma_j}(m_t,\sigma_t)\sigma_{t,j} \right].
$$
Using the update formulas \ref{algo-log}:
\begin{equation*}
    (m_{t+1,j}-m_j)^2 = (m_{t,j}-m_j)^2 + \eta_{t,j}^2 \sigma_{t,j}^4 \frac{\partial \bar{L}_t}{\partial m_j}(m_t,\sigma_t)^2 - 2 \eta_{t,j} \sigma_{t,j}^2 \frac{\partial \bar{L}_t}{\partial m_j}(m_t,\sigma_t) (m_{t,j}-m_j)
\end{equation*}
and 
\begin{equation*}
\sigma_{t+1,j}^2
= \sigma_{t,j}^2 + \frac{\eta_{t,j}^2 \sigma_{t,j}^4}{2} \frac{\partial \bar{L}_t}{\partial \sigma}_j(m_{t,j},\sigma_{t,j})^2 - \eta_{t,j} \sigma_{t,j}^2 \sqrt{1+\bigg(\frac{\eta_{t,j}\sigma_{t,j}\frac{\partial \bar{L}_t}{\partial \sigma_j}(m_t,\sigma_t)}{2}\bigg)^2} \frac{\partial \bar{L}_t}{\partial \sigma_j}(m_t,\sigma_t)\sigma_{t,j}.
\end{equation*}
Rearranging the terms, we get:
\begin{equation*}
    \frac{\partial \bar{L}_t}{\partial m_j}(m_t,\sigma_t)(m_{t,j}-m_j) = \frac{(m_{t,j}-m_j)^2-(m_{t+1,j}-m_j)^2}{2\eta_{t,j}\sigma_{t,j}^2} + \frac{\eta_{t,j}{ \sigma_{t,j}^2} \frac{ \partial \bar{L}_t}{\partial m_j}(m_t,\sigma_t)^2 }{2} 
\end{equation*}
and
\begin{equation*}
     \frac{\partial  \bar{L}_t}{\partial \sigma_j}(m_t,\sigma_t)\sigma_{t,j} = \frac{\sigma_{t,j}^2-\sigma_{t+1,j}^2}{\eta_{t,j}\sigma_{t,j}^2\sqrt{1+\bigg(\frac{\eta_{t,j}\sigma_{t,j}\frac{\partial \bar{L}_t}{\partial \sigma_j}(m_t,\sigma_t)}{2}\bigg)^2}} + \frac{\eta_{t,j} \sigma_{t,j}^2 \frac{ \partial \bar{L}_t}{\partial \sigma_j}(m_t,\sigma_t)^2 }{2\sqrt{1+\bigg(\frac{\eta_{t,j}\sigma_{t,j}\frac{\partial \bar{L}_t}{\partial \sigma_j}(m_t,\sigma_t)}{2}\bigg)^2}} .
\end{equation*}
We also use the boundedness of the gradients: 
for any $(m,\sigma) \in \mathcal{M}$, at any date $t$,
$$
\sum_{j=1}^d \left[ \frac{\partial \bar{L}_t}{\partial m_j}(m,\sigma)^2 + \frac{\partial \bar{L}_t}{\partial \sigma_j}(m,\sigma)^2 \right] \leq L^2.
$$
We upper bound the inverse of the square root by $1$, the gradient by $L$ and we sum over time:
\begin{align*}
    \sum_{t=1}^T  \bar{L}_t(\mu_t) -\bar{L}_t(\mu)
    & \leq \sum_{j=1}^d \sum_{t=1}^T \frac{(m_{t,j}-m_j)^2}{2} \bigg[ \frac{1}{\eta_{t,j} \sigma_{t,j}^2}-\frac{1}{\eta_{t-1,j} \sigma_{t-1,j}^2} \bigg] \\
    & + \sum_{j=1}^d \sum_{t=1}^T \frac{\eta_{t,j}\sigma_{t,j}^2}{2} \frac{\partial \bar{L}_t}{\partial m_j}(m_t,\sigma_{t})^2 \\
    & + \sum_{j=1}^d \sum_{t=1}^T \frac{\sigma_{t,j}^2}{2} \bigg[ \frac{2}{\eta_{t,j} \sigma_{t,j}^2} - \frac{2}{\eta_{t-1,j} \sigma_{t-1,j}^2} \bigg] \\
    & + \sum_{j=1}^d \sum_{t=1}^T \frac{\eta_{t,j}\sigma_{t,j}^2}{2} \frac{\partial \bar{L}_t}{\partial \sigma_j}(m_t,\sigma_t)^2 \\
    & = \sum_{j=1}^d \sum_{t=1}^T \frac{(m_{t,j}-m_j)^2}{2} \bigg[ \frac{1}{\eta_{t,j} \sigma_{t,j}^2}-\frac{1}{\eta_{t-1,j} \sigma_{t-1,j}^2} \bigg] \\
    & + \sum_{j=1}^d \sum_{t=1}^T \frac{\sigma_{t,j}^2}{2} \bigg[ \frac{2}{\eta_{t,j} \sigma_{t,j}^2} - \frac{2}{\eta_{t-1,j} \sigma_{t-1,j}^2} \bigg] \\
    & + \sum_{t=1}^T \frac{\eta_{t,j}\sigma_{t,j}^2}{2} \sum_{j=1}^d \left[ \frac{\partial \bar{L}_t}{\partial m_j}(m_t,\sigma_t)^2 + \frac{\partial \bar{L}_t}{\partial\sigma_j}(m_t,\sigma_t)^2 \right] \\
    & \leq \sum_{j=1}^d \sum_{t=1}^T \Biggl[ (m_{t,j}-m_j)^2 +\sigma_{t,j}^2\Biggr] \bigg[ \frac{1}{\eta_{t,j} \sigma_{t,j}^2}-\frac{1}{\eta_{t-1,j} \sigma_{t-1,j}^2} \bigg] \\
    & + \sum_{t=1}^T \frac{\eta_{t,j}\sigma_{t,j}^2}{2} \sum_{j=1}^d \left[ \frac{\partial \bar{L}_t}{\partial m_j}(m_t,\sigma_t)^2 + \frac{\partial \bar{L}_t}{\partial\sigma_j}(m_t,\sigma_t)^2 \right].
    \end{align*}
The key point in the following is that the difference
$$
\frac{1}{\eta_{t,j} \sigma_{t,j}^2}-\frac{1}{\eta_{t-1,j} \sigma_{t-1,j}^2}
$$
does not depend on $j$ on account of the formula $\eta_{t,j}=K/(\sqrt{t}\sigma_{t,j}^2)>0$.
We also recall that $$
\sum_{j=1}^d (m_{t,j}-m_j)^2+\sigma_{t,j}^2 \leq D^2.
$$

Moreover, 
$$\sum_{t=1}^{T} \frac{1}{\sqrt{t}} \leq 2 \sqrt{T} ,$$ 
so setting $\eta_{t,j}=\frac{K}{\sqrt{t}\sigma_{t,j}^2}>0$ with $K=\frac{D\sqrt{2}}{L}$, we finally have:
\begin{align*}
    \sum_{t=1}^T \bar{L}_t(\mu_t) -\bar{L}_t(\mu)
    & \leq \frac{1}{K}  \sum_{t=1}^T ( \sqrt{t} - \sqrt{t-1} ) \sum_{j=1}^d [(m_{t,j}-m_j)^2+\sigma_{t,j}^2] + \sum_{t=1}^T \frac{K}{\sqrt{t}} L^2 \\
    & \leq \frac{D^2}{K}  \sum_{t=1}^T ( \sqrt{t} - \sqrt{t-1} ) + \frac{KL^2}{2} \sum_{t=1}^T \frac{1}{\sqrt{t}}  \\
    & = \left( \frac{D^2}{K} + \frac{KL^2}{2} \right) \sqrt{T} \\
    & = DL \sqrt{2T},
    \end{align*}
where $K$ is chosen so that it minimizes the bound.

\subsection*{Strongly convex case:}
Now, we assume that each function $\bar{L}_t$ is $H$-strongly convex, for all $m\in \mathcal{M}_m$ and $\mu=(m,0)$:
\begin{align*}
    \bar{L}_t(\mu_t)  -\bar{L}_t(\mu) & \leq \nabla \bar{L}_t(\mu_t) ^T (\mu_t-\mu) - \frac{H}{2} \|\mu_t-\mu\|^2 \\
    & = \sum_{j=1}^d \bigg[ \frac{\partial \bar{L}_t}{\partial m_j}(m_t,\sigma_t)(m_{t,j}-m_j) + \frac{\partial \bar{L}_t}{\partial \sigma_j}(m_t,\sigma_t)\sigma_{t,j} - \frac{H}{2} (m_{t,j}-m_j)^2 - \frac{H}{2} \sigma_{t,j}^2 \bigg].
    \end{align*}
Again,
\begin{equation*}
    \frac{\partial \bar{L}_t}{\partial m_j}(m_t,\sigma_t)(m_{t,j}-m_j) = \frac{(m_{t,j}-m_j)^2-(m_{t+1,j}-m_j)^2}{2\eta_{t,j}\sigma_{t,j}^2} + \frac{\eta_{t,j} \sigma_{t,j}^2 \frac{ \partial \bar{L}_t}{\partial m_j}(m_t,\sigma_t)^2 }{2} 
\end{equation*}
and
\begin{equation*}
    \frac{\partial \bar{L}_t}{\partial \sigma_j}(m_t,\sigma_t)\sigma_{t,j} =  \frac{\sigma_{t,j}^2-\sigma_{t+1,j}^2}{\eta_{t,j}\sigma_{t,j}^2\sqrt{1+\bigg(\frac{\eta_{t,j}\sigma_{t,j}\frac{\partial \bar{L}_t}{\partial \sigma_j}(m_t,\sigma_t)}{2}\bigg)^2}} + \frac{\eta_{t,j} \sigma_{t,j}^2 \frac{ \partial \bar{L}_t}{\partial \sigma_j}(m_t,\sigma_t)^2 }{2\sqrt{1+\bigg(\frac{\eta_{t,j}\sigma_{t,j}\frac{\partial \bar{L}_t}{\partial \sigma_j}(m_t,\sigma_t)}{2}\bigg)^2}} ,
\end{equation*}
and then as previously with $\eta_{t,j}=\frac{2}{Ht\sigma_{t,j}^2}$:
\begin{align*}
    \sum_{t=1}^T \bar{L}_t(\mu_t) -\bar{L}_t(\mu)
    & \leq \sum_{j=1}^d \sum_{t=1}^T \frac{(m_{t,j}-m_j)^2}{2} \Biggl[ \frac{1}{\eta_{t,j} \sigma_{t,j}^2} -\frac{1}{\eta_{t-1,j} \sigma_{t-1,j}^2} - H \Biggr] \\
    & + \sum_{j=1}^d \sum_{t=1}^T \frac{\eta_{t,j}\sigma_{t,j}^2}{2} \frac{\partial \bar{L}_t}{\partial m_j}(m_t,\sigma_t)^2 \\
    & + \sum_{j=1}^d \sum_{t=1}^T \frac{\sigma_{t,j}^2}{2} \bigg[ \frac{2}{\eta_{t,j} \sigma_{t,j}^2} - \frac{2}{\eta_{t-1,j} \sigma_{t-1,j}^2} - H \bigg] \\
    & + \sum_{j=1}^d \sum_{t=1}^T \frac{\eta_{t,j}\sigma_{t,j}^2}{2} \frac{\partial \bar{L}_t}{\partial \sigma_j}(m_t,\sigma_t)^2 \\
    & \leq \sum_{j=1}^d \sum_{t=1}^T \frac{(m_{t,j}-m_j)^2}{2} \bigg[ \frac{tH}{2} - \frac{(t-1)H}{2} - H \bigg] \\
    & + \sum_{j=1}^d \sum_{t=1}^T \frac{\sigma_{t,j}^2}{2} \bigg[ tH - (t-1)H - H \bigg] \\
    & + \sum_{t=1}^T \frac{1}{Ht} \sum_{j=1}^d \left[ \frac{\partial \bar{L}_t}{\partial m_j}(m_t,\sigma_t)^2 + \frac{\partial \bar{L}_t}{\partial \sigma_j}(m_t,\sigma_t)^2  \right] \\
    & \leq \sum_{j=1}^d \sum_{t=1}^T \frac{(m_{t,j}-m_j)^2}{2} \bigg[ \frac{H}{2} - H \bigg] + 0 + \sum_{t=1}^T \frac{L^2}{Ht} \\
    & \leq \frac{L^2}{H}(1+\log(T)) ,
    \end{align*}
which ends the proof.
\end{proof}

\begin{proof}[Proof of Theorem~\ref{thm2}]
The proof is exactly the same as for Theorem \ref{thm1}.
As previously, we first prove by recursion on $T$ that
\begin{equation}
 \label{eq-lemma23-shalev-prime-bis}
 \forall\mu\in \mathcal{M}\text{, } \sum_{t=1}^T \mu_{t+1}^T \nabla \bar{L}_t(\mu_t)
 \leq  \sum_{t=1}^T \mu^T \nabla \bar{L}_t(\mu_t) + \frac{\|\mu-\mu_1\|^2}{\eta}.
\end{equation}
It is obvious that it holds for $T=0$. Assume now that~\eqref{eq-lemma23-shalev-prime-bis} holds for some integer $T-1$.
Then for all $\mu\in M$,
\begin{align*}
\sum_{t=1}^{T} \mu_{t+1}^T \nabla \bar{L}_t(\mu_t) & = \sum_{t=1}^{T-1} \mu_{t+1}^T \nabla \bar{L}_t(\mu_t) + \mu_{T+1}^T \nabla \bar{L}_T(\mu_T)
 \\
 & \leq  \sum_{t=1}^{T-1} \mu^T \nabla \bar{L}_t(\mu_t) + \frac{\|\mu-\mu_1\|^2}{\eta} + \mu_{T+1}^T \nabla \bar{L}_T(\mu_T)
\end{align*}
as ~\eqref{eq-lemma23-shalev-prime-bis} holds for $T-1$. Apply this again to $\mu=\mu_{T+1}$:
\begin{align*}
\sum_{t=1}^{T} \mu_{t+1}^T \nabla \bar{L}_t(\mu_t) & \leq \sum_{t=1}^{T} \mu_{T+1}^T \nabla \bar{L}_t(\mu_t) + \frac{\|\mu-\mu_1\|^2}{\eta}
\\
& = \min_{m\in \mathcal{M}} \left[\sum_{t=1}^T m^T \nabla \bar{L}_t(\mu_t) + \frac{\|\mu-\mu_1\|^2}{\eta}\right] \text{, by definition of } \mu_{T+1}
\\
& \leq \sum_{t=1}^T \mu^T \nabla \bar{L}_t(\mu_t) + \frac{\|\mu-\mu_1\|^2}{\eta}
\end{align*}
for all $\mu\in \mathcal{M}$. Thus,~\eqref{eq-lemma23-shalev-prime-bis} holds for $T$, and thus for integers.

We prove now that for any $t\in\mathbb{N}$,
\begin{equation}
 \label{eq-FTRL-bis}
  \mu_t^T \nabla \bar{L}_t(\mu_t) - \mu_{t+1}^T \nabla \bar{L}_t(\mu_t)
 \leq \eta L^2.
\end{equation}
Indeed,
\begin{align}
 \mu_t^T  \nabla \bar{L}_t(\mu_t) - \mu_{t+1}^T \nabla \bar{L}_t(\mu_t)
 & = (\mu_t - \mu_{t+1})^T \nabla \bar{L}_t(\mu_t)
  \nonumber
 \\
 & \leq \|\mu_t-\mu_{t+1}\| \|\nabla \bar{L}_t(\mu_t)\|
  \nonumber
 \\
 & \leq L \|\mu_t - \mu_{t+1}\|
 \label{eq-FTRL-step1-bis}
\end{align}
as previously. Define
$$
G_{t}(\mu) = \sum_{i=1}^{t-1} \mu^T \nabla \bar{L}_t(\mu_i) + \frac{\|\mu-\mu_1\|^2}{\eta}.
$$
Obviously, $G_t$ is $1/\eta$-strongly convex: for any $(\mu,\mu')$,
$$
G_{t}(\mu') - G_{t}(\mu) \geq (\mu'-\mu)^T \nabla G_t(\mu) + \frac{ \|\mu'-\mu\|^2}{2 \eta} .
$$
In particular, $\mu_t$ is a minimizer of $G_t$:
$$
G_{t}(\mu_{t+1}) - G_{t}(\mu_t) \geq \frac{ \|\mu_{t+1}-\mu_t\|^2}{2 \eta}.
$$
Similarly,
$$
G_{t+1}(\mu_{t}) - G_{t+1}(\mu_{t+1}) \geq \frac{ \|\mu_{t+1}-\mu_t\|^2}{2 \eta}.
$$
Hence:
$$
\bar{L}_{t}(\mu_t)-\bar{L}_{t}(\mu_{t+1}) \geq \frac{\|\mu_{t+1}-\mu_t\|^2}{\eta},
$$
and then
\begin{align*}
 \|\mu_{t+1}-\mu_t\|
 \leq \sqrt{\eta\left[ \mu_t^T \nabla \bar{L}_t(\mu_t) - \mu_{t+1}^T \nabla \bar{L}_t(\mu_t)\right]}
\end{align*}
which combined with~\eqref{eq-FTRL-step1-bis} leads to~\eqref{eq-FTRL-bis}.

Finally, as for Theorem \ref{thm1}:
\begin{equation*}
\sum_{t=1}^{T}  \loss_t(\hat{\theta}_t)
-  \sum_{t=1}^{T} \mathbb{E}_{\theta\sim q_\mu}[\loss_t(\theta)]
\leq \eta T L^2 + \frac{\|\mu-\mu_1\|^2}{\eta},
\end{equation*}
which ends the proof.
\end{proof}

\end{document}